\documentclass[a4paper,UKenglish,cleveref, autoref, thm-restate]{lipics-v2021}

\usepackage{amsmath,amsthm,bm}
\usepackage{amssymb,bbm}
\usepackage{amsfonts,amsmath}
\usepackage{stackrel}
\usepackage{xcolor,float}
\usepackage{algorithmic}
\usepackage{multirow}
\usepackage{amsmath,amssymb,amsthm, mathtools}
 \newtheorem{assumption}[theorem]{Assumption}
\usepackage[ruled]{algorithm2e}
\usepackage{mathrsfs}

\newcommand{\E}{\mathbb{E}}

\newcommand{\R}{\mathbb{R}}
\newcommand{\Prob}{\mathbb{P}}
\newcommand{\indep}{\perp \!\!\! \perp}
\newcommand{\bE}{\mathbb{E}}
\newcommand{\Cov}{{\rm Cov}}
\newcommand{\1}{{\mathbbm{1}}}

\newcommand{\bN}{\mathbb{N}}
\newcommand{\bP}{\mathbb{P}}

\newcommand{\bR}{{\mathbb R}}

\newcommand{\sO}{{\mathscr{O}}}
\renewcommand{\cal}{\mathcal}
\newcommand\cA{{\mathcal A}}

\newcommand{\cC}{{\cal C}}
\newcommand{\Cint}{\mathcal{I}}

\newcommand{\cY}{{\cal Y}}
\newcommand{\cD}{{\cal D}}

\newcommand{\cG}{{\cal G}}
\newcommand\cH{{\mathcal H}}

\newcommand{\cZ}{{\cal Z}}
\newcommand{\cL}{{\cal L}}

\newcommand{\cN}{{\cal N}}
\newcommand\cO{{\mathcal O}}
\newcommand{\cP}{{\cal P}}
\newcommand{\cQ}{{\cal Q}}

\newcommand{\cU}{{\mathcal U}}

\newcommand\cW{{\mathcal W}}
\newcommand{\cX}{{\mathcal X}}
\renewcommand{\1}{\mathbbm{1}}	%

\DeclareMathOperator*{\argmin}{arg\,min}

\newcommand{\remove}[1]{}

\newcommand{\GMC}{$s$-Happy Multi-calibration}

\newcommand{\gmc}{$s$-happily multicalibrated}
\newcommand{\hmd}{happily multicalibrated}
\newcommand{\HM}{\textit{HappyMap}}

\renewcommand{\HM}{HappyMap}

\title{\HM : A Generalized Multicalibration Method}

\author{Zhun Deng}{Department of Computer Science, Columbia University, USA \and \url{https://www.zhundeng.org} }{zhundeng@g.harvard.edu}{}{}

\author{Cynthia Dwork}{John A. Paulson School of Engineering and Science, Harvard University, USA  \and \url{https://dwork.seas.harvard.edu/}}{dwork@seas.harvard.edu
}{}{}

\author{Linjun Zhang \thanks{Author names are listed in alphabetical order.}}{Department of Statistics, Rutgers University, USA  \and \url{https://linjunz.github.io/}}{linjun.zhang@rutgers.edu
}{https://orcid.org/0000-0002-8309-7164}{}

\funding{This work was supported in part by Alfred P. Sloan Foundation G-2020-13941, Simons Foundation Collaborative Research Grant 733782, and NSF DMS-2015378.}

\authorrunning{Z. Deng, C. Dwork, and L. Zhang } 

\Copyright{Zhun Deng, Cynthia Dwork, and Linjun Zhang} 



\category{} 

\relatedversion{} 

\acknowledgements{We thank all the reviewers for their comments and suggestions.  We are indebted to Aaron Roth for his insightful and valuable feedback, which greatly improved the paper.}

\EventEditors{Yael Tauman Kalai}
\EventNoEds{1}
\EventLongTitle{14th Innovations in Theoretical Computer Science Conference (ITCS 2023)}
\EventShortTitle{ITCS 2023}
\EventAcronym{ITCS}
\EventYear{2023}
\EventDate{January 10--13, 2023}
\EventLocation{MIT, Cambridge, Massachusetts, USA}
\EventLogo{}
\SeriesVolume{251}
\ArticleNo{54}

\begin{document}

\maketitle

\begin{abstract}
Multi-calibration is a powerful and evolving concept originating in the field of algorithmic fairness. For a predictor $f$ that estimates the outcome $y$ given covariates $x$, and for a function class $\cC$, multi-calibration requires that the predictor $f(x)$ and outcome $y$ are indistinguishable under the class of auditors in $\cC$.  Fairness is captured by incorporating demographic subgroups into the class of functions~$\cC$. Recent work has shown that, by enriching the class $\cC$ to incorporate appropriate propensity re-weighting functions, multi-calibration also yields target-independent learning, wherein a model trained on a source domain performs well on unseen, future, target domains {(approximately) captured by the re-weightings.}
 
 Formally, multi-calibration with respect to $\cC$ bounds
$  \big|\bE_{(x,y)\sim \cD}[c(f(x),x)\cdot(f(x)-y)]\big|$  for all $c \in \cC$.
 In this work, we view the term $(f(x)-y)$ as just one specific mapping, and explore the power of an enriched class of mappings. We propose \textit{\GMC}, a generalization of multi-calibration, which yields a wide range of new applications, including a new fairness notion for uncertainty quantification (conformal prediction),
a novel technique for conformal prediction under covariate shift, and a different approach to
analyzing missing data, while also yielding a unified understanding of several existing seemingly disparate algorithmic fairness notions and target-independent learning approaches. 

We give a single \textit{HappyMap} meta-algorithm that captures all these results, together with a sufficiency condition for its success. 
\end{abstract}

\section{Introduction}

Prediction algorithms score individuals or individual instances, assigning to each a score in~$[0,1]$ typically interpreted as a probability, for example, the probability that it will rain tomorrow.  The predictions are calibrated if, for all $v \in [0,1]$, among those instances assigned the value~$v$, a $v$ fraction have a positive outcome. Calibration has been viewed as the {\it sine qua non} of prediction for  decades~\cite{dawid1982well}.  

The requirement that a predictor be simultaneously calibrated on each of two or more disjoint groups, meaning, it is calibrated on each group when viewed in isolation, was first proposed as a {\em fairness} condition by Kleinberg, Mullainathan, and Raghavan~\cite{KMR}.  Inspired by~\cite{KMR} and in an attempt to bridge the gap between {\em individual fairness}~\cite{FtA}, which demands that similar individuals be treated similarly but requires task-specific measures of similarity, and {\em group fairness}, which can be specious~\cite{FtA}, H\'{e}bert-Johnson, Kim, Reingold, and Rothblum proposed {\em multicalibration}, which requires calibration on a (possibly large) pre-specified collection of {\em arbitrarily intersecting} sets that can be identified within a specified class of computations~\cite{hebert2018multicalibration}. A related, independent, work of Kearns, Neel, Roth, and Wu considered the analogous setting but for Boolean-valued classifiers. {They argued that including intersecting groups can prevent fairness gerrymandering, and developed  multi-parity~\cite{kearns2018preventing}.}

The area has blossomed in theory and in application.  For example, multicalibration has been used for fair ranking~\cite{DKRRY2019}, and for providing an indistinguishability-based interpretation of individual probabilibilities, {\it i.e.}, probabilities for non-repeatable events~\cite{DKRRY2021}.  Multicalibration has also been shown to yield {\em omniprediction}, meaning that for every ``nice" loss function $\ell$, the scores assigned by the multicalibrated predictor can be post-processed, with no additional training, to be competitive, on the training distribution, with the best predictor in~$\cC$~\cite{gopalan2021omnipredictors}.  
\begin{definition}[Multicalibration \cite{hebert2018multicalibration} as presented in \cite{kim2022universal}]
\label{def:multicalibration}
Let $\cC\subseteq\{[0,1]\times \cX\rightarrow \bR\}$ be a collection of functions. For a given distribution supported on $\cX\times \cY$, a predictor $f:\cX\mapsto[0,1]$ is $(\cC,\alpha)$-multicalibrated over $\cD$ if $\forall c \in \cC$:
\begin{equation}\label{eq:1 intro}
  \Big|\bE_{(x,y)\sim \cD}[c(f(x),x)\cdot(f(x)-y)]\Big|\le \alpha.
\end{equation}
\end{definition}
Fairness is captured by incorporating demographic subgroups into the class of functions~$\cC$. By enriching the class $\cC$ to incorporate appropriate propensity re-weighting functions, multicalibration also yields target-independent learning, wherein a model trained on a source domain performs well on unseen, future, target domains {(approximately) captured by the re-weightings.}~\cite{kim2022universal}.

In this work, we view the term $({f}(x)-y)$ in Equation~\eqref{eq:1 intro} as just one specific mapping $s(f,y): \mathbb{R} \times \mathcal{Y} \rightarrow \mathbb{R}$, and explore the power of an enriched class of mappings. To this end, we propose \textit{HappyMap}, a generalization of multicalibration, which yields a wide range of new applications, including a new algorithm for fair uncertainty quantification,
a novel technique for conformal prediction under distributional (a.k.a. covariate) shift, and a different approach to analyzing missing data, while also yielding a unified understanding of several existing seemingly disparate algorithmic fairness notions and target-independent learning approaches. 

We give a single \textit{HappyMap} meta-algorithm that captures all these results, together with a sufficiency condition for its success. Roughly speaking, the requirement is that the mapping have an anti-derivative satisfying a smoothness-like assumption (see Section~\ref{sec:notion}).  We say such a mapping is {\em happy}.  Loosely speaking, the anti-derivative serves as a potential function, which yields an upper bound on the number of iterations of our algorithm.

\medskip
\noindent
{\bf Summary of Contributions.}

\begin{enumerate}
\item  We propose \textit{HappyMap}, a generalization of multicalibration, by enriching the class of mappings (alternatives for the term $(f(x)-y)$ in Equation~\ref{eq:1}), as discussed above, and provide a \textit{HappyMap} meta-algorithm having comparable running time and sample complexity to the other multicalibration algorithms in the literature, and give sufficient conditions for its success (Section~\ref{sec:notion}). 
\item We demonstrate the flexibility of \textit{HappyMap} by first applying it to obtain generalized versions of {\em fair uncertainty quantification}~\cite{romano2019malice} (Section~\ref{sec:application}). 
\item Furthermore, we apply \textit{HappyMap} to problems in {\em target-independent} learning that lie beyond the statistical estimation problems considered in \cite{kim2022universal}, obtaining target-independent statistical inference and uncertainty quantification.  Our approach also yields a fruitful new perspective on analyzing missing data, giving new solutions to this problem (Section~\ref{sec:ua}).

\end{enumerate}

\noindent{\bf A High-Level Perspective.}  The seminal work of H\'{e}bert-Johnson {\it et al.} \cite{hebert2018multicalibration} has long tantalized us with the suggestion of a new paradigm for machine learning.  We coin the term {\em micro-learning} to describe gradient descent, a learning paradigm paradigm based on loss minimization, consisting of a sequence of local model updates.  In contrast, the multiaccuracy and multicalibration algorithms of~\cite{hebert2018multicalibration} produce predictors via interactions (through weak agnostic learning) with auditors who find large problem areas without the intercession of a loss function, and make correspondingly large model updates, suggested to us what we call {\em macro-learning}. This perspective has a compelling transparency story, \textit{``the learning algorithm finds large errors and fixes them"}, which is particularly appealing when the auditors are interrogating the treatment of demographic subgroups.  The fact that macro-learning can be used as post-processing~\cite{hebert2018multicalibration} tells us that the two paradigms can work in concert.
The concept of \GMC{} proposed in this paper advances the broad vision of macro-learning.  

\section{Preliminaries}

\subsection{Notation}

For $d\in\bN+$, any convex set $S\subseteq \bR^d$ and vector $v\in\bR^d$, we use $\Pi_S(v)$ to denote the the projection of $v$ on $S$ under Euclidean distance. We sometimes use the probability notion $\bP$ also for density function, for instance $\bP(x)$ means the density function evaluated at $x$ for continuous random vector $x$.  Let us denote $X\in \mathcal X$ as the feature vector (typically $\mathcal X=\R^d$), $Y\in\mathcal Y$ (typically $\mathcal Y=\R$ for regression problems and $\mathcal Y=\{0,1\}$ for classification problems) as the response that we are trying to predict. For a joint distribution of $(x,y)\in\mathcal X\times\mathcal Y$, let us denote the marginal distribution of $x$ and $y$ as $\cD^{\mathcal X}$ and $\cD^{\mathcal Y}$ respectively.
 For two positive sequences $\{a_k\}$ and $\{b_k\}$, we write $a_k =\cO(b_k)$ (or $a_n\lesssim b_n$), and $a_k = o(b_k)$, if $\lim_{k\rightarrow\infty}(a_k/b_k) < \infty$ and $\lim_{k\rightarrow\infty}(a_k/b_k) =0$ respectively. $\tilde {\cO}(\cdot)$ denotes the term, neglecting the logarithmic factors. We also write $a_k=\Theta(b_k)$ (or $a_k\asymp b_k$) if $a_k\lesssim b_k$ and $b_k\lesssim a_k$. We use $\cO_p(\cdot)$ to denote stochastic boundedness: a random variable $Z_k=\cO_p(a_k)$ for some real sequence $a_k$ if $\forall \epsilon>0$, there exists $M,N>0$ such that if $k>N$, $\bP(|Z_k/a_k|>M)\le \epsilon$. For two numbers $a<b$, we use the notation $U(a,b)$ to denote the uniform distribution on $[a,b]$. For $\mu\in\R$ and $\sigma>0$, we use $N(\mu,\sigma^2)$ to denote a normal distribution with mean $\mu$ and variance $\sigma^2$.

\section{\GMC}\label{sec:notion}

We now formally state our generalization of multicalibration.
\begin{definition}[\GMC ]\label{def:GMC}
 Let $\cC\subseteq\{\bR\times \cX\rightarrow \bR\}$ be a collection of functions. For a given distribution supported on $\cX\times \cY$ and a mapping $s:\bR\times\cX\mapsto\bR$, a predictor $f:\cX\mapsto \bR$ is $(\cC,\alpha)$-\gmc~over $\cD$ and $s$ if for all $c\in \cC$:
\begin{equation}\label{eq:1}
  \Big|\bE_{(x,y)\sim \cD}[c(f(x),x)\cdot s(f(x),y)]\Big|\le \alpha.
\end{equation}
\end{definition}
\noindent
When $\cD$ and $s$ are clear from the context, we will also simply say $f$ is $(\cC,\alpha)$-\hmd. Sometimes we will constrain the range of $f$ to a certain convex set $O$, for example, $[0,1]$; in this case we say $f:\cX\mapsto O$ is $(C,\alpha)$-\gmc{} when it satisfies Eq.~\eqref{eq:1}. When $c(f(x),x)$ is independent of $f$ and only depends on $x$, we will simply write $c(x)$.

\GMC{} (Equation~\eqref{eq:1}) captures several multi-group fairness notions in the literature. For example,  if we let $c(f(x),x)$ be be independent of $f$, and take $s(f(x),y)=f(x)-y$, we recover multi-accuracy \cite{kim2019multiaccuracy,hebert2018multicalibration};  
if we take $c(f(x),x)=\tilde c(x)\cdot w(f(x))$ for some functions $\tilde c$ and $w$ and $s(f(x),y)=f(x)-y$, we recover the low-degree multicalibration notion in~\cite{gopalan2022low}; if we take $s(f(x),y)$ to be a non-negative loss function and $c(x)$ to be indicator functions of different groups, we recover the minimax group fairness in \cite{diana2021minimax}. 

\noindent{\bf The \textit{HappyMap} Meta-Algorithm.} We now describe the \textit{HappyMap} meta-algorithm and prove its key properties.  For simplicity, we describe the \textit{population} version of our algorithm, where we are allowed to access $\bE_{(x,y)\sim \cD}[c(f(x),x)s(f(x),y)]$. 
Of course, in practice, we will need to estimate these quantities from training data, and so we describe the sample version of Algorithm~\ref{alg:gmc} in Appendix~\ref{app:sample version} and derive the corresponding required sample complexity. In keeping with the multi-group fairness literature \cite{hebert2018multicalibration,kim2019multiaccuracy}, we can either use a fresh sample per iteration or, when samples are limited, apply techniques from adaptive data analysis \cite{dwork2015generalization,dwork2015reusable, dwork2015preserving} to re-use the same samples in each iteration, as suggested in~\cite{hebert2018multicalibration}. 

 We consider the general case and aim to return a $f:\cX\mapsto O$, for a given convex set $O$, which is $(\cC,\alpha)$-\hmd . Without loss of generality, let us assume the class $\cC$ is symmetric in the sense that $c$ and $-c$ are both in $\cC$ (we can always augment $\tilde \cC$ by including $-c$ for $c\in\tilde\cC$, so this is without loss of generality).
The algorithm is invoked with an initial predictor $f_0$, which can be  trivial ($f_0(x)=c$ for all $x \in \cX$) or can be an artisanal predictor imbued with extensive domain knowledge.

\begin{algorithm}[H]{
\caption{HappyMap}\label{alg:gmc}
\KwIn{{Tolerance $\alpha>0$,} step size $\eta>0$, bound $T\in\bN_+$ on number of iterations, initial predictor $f_0(\cdot)$, distribution $\cD$, convex set $O\subseteq \bR$, mapping $s:\bR \times {\mathcal Y}\rightarrow \bR$}
}
Set $t=0$

\While {$t < T$ ~and $~ \exists c_t\in\cC: \bE_{(x,y)\sim \cD}[c_t(f_t(x),x)s(f_t(x),y)]> \alpha$}
{
\vspace{0.05in}
Let $c_t$ be an arbitrary element of $\cC$ satisfying the condition in the while statement

$\forall x\in \cX$,~$f_{t+1}(x)=\Pi_O\big[f_t(x)-\eta\cdot c_t(f(x),x)\big]$ 
\\
$t=t+1$}

\KwOut{$f_t(\cdot)$}
\end{algorithm}

{
As in previous work \cite{hebert2018multicalibration,romano2019malice,kim2022universal,kearns2018preventing}, we do not address the complexity of the weak  learner whose job it is to search for functions $c_t \in \cC$ satisfying the condition of the while loop, or to report that none exists.  The problem is at least as hard as agnostic learning~\cite{hebert2018multicalibration,kearns2018preventing}. See~\cite{burhanpurkar2021scaffolding} for discussion of this issue and its implications for fairness.} {Henceforth, our unit of  computational complexity will be an invocation of the weak learner, {\it i.e.}, an iteration of the while loop.}

\noindent{\bf Analysis.} Theorem~\ref{thm:gmc} below summarizes the theoretical guarantees for Algorithm~\ref{alg:gmc}.
As is common in the literature, the proof of termination relies on a potential function argument (see, {\it e.g.}, \cite{bubeck2015convex,hebert2018multicalibration}). 
The key new ingredient is the notion of a {\em happy} mapping, which provides the conditions under which we can find the necessary potential function.

We will require some assumptions, which we state and discuss before stating the theorem.

\noindent{\bf Three Assumptions.}
(a) For all $c \in \cC, x\in\cX$, $\bE_{x\sim \cD^\cX}[c^2(f(x),x)]\le B$, where $\cD^\cX$ is the marginal distribution of $\cD$ on $x$; (b) There exists a potential function $\cL:\bR\times\cY\mapsto \bR$, constants $C^l_\cL, C^u_\cL$, positive constant $\kappa_\cL>0$, such that (1) for all $f,\tilde{f}:\cX\mapsto \bR$, $C^l_\cL\le\E_{(x,y)\sim\cD}[\cL(f(x),y)]$, (2)  $C^l_\cU\ge\E_{(x,y)\sim\cD}[\cL(f_0(x),y)]$ for the initialization $f_0$, and (3) $\E_{(x,y)\sim\cD}[\cL(f(x),y)-\cL(\tilde{f}(x),y)]\ge \E_{(x,y)\sim\cD}[(f(x)-\tilde{f}(x))s(f(x),y)]-\kappa_\cL\E_{x\sim\cD^\cX}[(f(x)-\tilde{f}(x))^2]$; (c) For all $y\in \cY$ and $v\in\bR$, $\cL(\Pi_{O}[v],y)\le \cL(v,y)$. 

Assumption (a) is routine and
Assumption (c) says that the potential function decreases upon projection with respect to its first coordinate. One concrete example is the case in which $O=\cY=[0,1]$ and $\cL(v,y)=|v-y|^2$. 

We now turn to Assumption~(b), focusing on \textit{how to construct the potential function $\cL$}. First note that the assumption is closely related to {\em smoothness}, a widely used concept in optimization. We use the following fact.

\noindent\textbf{Fact.} 
If $\cL(\cdot,\cdot):\bR\times\cX\mapsto \bR$ is $2\kappa_\cL$-smooth with respect to its first coordinate, i.e. $\cL(\cdot,\cdot)$ is continuously differentiable with respect to its first coordinate, and the corresponding partial derivative is $2\kappa_\cL$-Lipchitz, then, for all $x\in\cX, y\in\cY$ and $f,\tilde{f}:\cX\mapsto \bR$,
\begin{equation}\label{eq:smoothness}
    \cL(f(x),y)-\cL(\tilde{f}(x),y)\ge \partial_u \cL(u,y)|_{u=f(x)}(f(x)-\tilde{f}(x))-\kappa_\cL(f(x)-\tilde{f}(x))^2.
    \end{equation}
\noindent
Thus, if $s(f(x),y)$ is differentiable with respect to its first coordinate and $|\partial_us(u,y)|\le \kappa_\cL$, we can take the potential function to be the anti-derivative of $s$, specifically, $\cL(f(x),y)=\int_{g}^{f(x)}s(u,y)du$ for any $g\in \bR$, as long as we can ensure $\int_{g}^{f(x)}s(u,y)du\ge C^l_\cL$ for all $f(x)$ and $y$. Then, taking expectation over both sides of this equation and \eqref{eq:smoothness}, we can satisfy (b). In fact, even when the function $s(u,y)$ is not differentiable everywhere with respect to $u$, we can still follow the general idea above to construct $\cL$, and assumption (b) can still be satisfied with respect to the expectation of $\cL$. For example, in Section~\ref{subsec:conformal}, when $s(u,y)=\1\{u\le y\}-(1-\delta)$, taking $\cL(f(x),y)=(1-\delta)\cdot f(x)-\min(f(x)-y,0)$ will satisfy the desired condition.

 \begin{theorem}\label{thm:gmc}
Under Assumptions (a)-(c) above, and if $\cC$ is symmetric\footnote{Recall that symmetry is without loss of generality.}, then  Algorithm~\ref{alg:gmc} with a suitably chosen $\eta=\cO(\alpha/(\kappa_\cL B))$ converges in $T=\cO((C^u_\cL-C^l_\cL) \kappa_\cL B/\alpha^2)$ iterations and outputs a function $f$ satisfying $$\Big|\bE_{(x,y)\sim \cD}[c(f(x),x)s(f(x),y)]\Big|\le\alpha.$$
\end{theorem}

\begin{proof}[Proof of Theorem~\ref{thm:gmc}]
First, since the class $\cC$ is symmetric, that means as long as we can prove for all $c\in\cC$, $$\bE_{(x,y)\sim \cD}[c(f_t(x),x)s(f_t(x),y)]\le\alpha,$$
for the output $f_t(\cdot)$ of Algorithm~\ref{alg:gmc}, we will also have for all $c\in\cC$, $$\Big|\bE_{(x,y)\sim \cD}[c(f_t(x),x)s(f_t(x),y)]\Big|\le\alpha.$$

By our assumption, there exists a potential function $\cL$, a constant $C_\cL$, and a non-negative constant $\kappa_\cL$, such that for any $x\in\cX$, 
\begin{align*}
 \bE_{(x,y)\sim\cD}\cL(f_t(x),y)-\bE_{(x,y)\sim\cD}\cL(f_{t+1}(x),y)&\ge \bE_{(x,y)\sim\cD}(f_t(x)-f_{t+1}(x))s(f_t(x),y)\\
 &-\bE_{(x,y)\sim\cD}\kappa_\cL(f_t(x)-f_{t+1}(x))^2.   
\end{align*}

As a result, we have
\begin{align*}
&\bE_{(x,y)\sim\cD}\cL(f_t(x),y)-\bE_{(x,y)\sim\cD}\cL(f_{t+1}(x),y)\\
\ge& \bE_{(x,y)\sim\cD}\cL(f_t(x),y)-\bE_{(x,y)\sim\cD}\cL\big(f_{t}(x)-\eta c_t(f_t(x),x),y\big) \\
\ge& \bE_{(x,y)\sim\cD}\eta c_t(f_t(x),x)\cdot s(f_t(x),y)-\kappa_\cL\bE_{(x,y)\sim\cD}(\eta c_t(f_t(x),x))^2.
\end{align*}
The first inequality is because of (c) in  our assumption and the second inequality is because of (b) in our assumption.

Given $\bE_{x\sim \cD_\cX}[c^2(f(x),x)]\le B$, if there exists $c_t\in\cC$, $\bE_{(x,y)\sim \cD}[c_t(f_t(x),x)s(f_t(x),y)]>\alpha$
\begin{align*}
&\bE_{(x,y)\sim \cD}\cL(f_t(x),y)-\bE_{(x,y)\sim \cD}\cL(f_{t+1}(x),y)\\
\ge &\eta\bE_{(x,y)\sim \cD} c_t(f_t(x),x)\cdot s(f_t(x),y)
-\kappa_\cL\bE_{(x,y)\sim \cD}(\eta c_t(f_t(x),x))^2\ge \eta \alpha - \kappa_\cL\eta^2B.
\end{align*}

Take $\eta=\alpha/(2\kappa_\cL B)$, we have 
$$\bE_{(x,y)\sim \cD}\cL(f_t(x),y)-\bE_{(x,y)\sim \cD}\cL(f_{t+1}(x),y)\ge \frac{\alpha^2}{2\kappa_\cL B}.$$

Since $C^u_\cL\ge\bE_{(x,y)\sim\cD}\cL(f_0(x),y)$ and $\bE_{(x,y)\sim\cD}\cL(f(x),y\ge C^l_\cL$ for all $f$ by assumption, each update will result in a progress at least $\frac{\alpha^2}{2\kappa_\cL B}$ for each iteration if it happens, we know there are at most $2\kappa_\cL B(C^u_\cL-C^l_\cL)/\alpha^2$ updates.
\end{proof}

In the following sections, we will always assume $\cC$ is symmetric. Also, for simplicity, we assume $\cC$ is closed with respect to $L_2$-norm.

\section{Application: Algorithmic Fairness in Prediction Intervals}\label{sec:application}

{\em Conformal prediction} is a popular approach to uncertainty quantification in prediction models. Continuing in this vein, Romano {\it et al.} proposed a new group fairness criterion, {\em equalized coverage}~\cite{romano2019malice}, in which the goal is to construct a prediction interval $\Cint(x)$ that covers $y$ with comparable probability across all protected groups of interest.  More precisely, given a collection of disjoint protected demographic groups $\mathcal A\subset 2^\mathcal X$, the set-valued function $\Cint: \cX \rightarrow \bR$ provides {\em equalized coverage} if
\begin{equation}\label{def:equalized coverage: orginal}
\Prob(y\in \Cint(x)\mid x\in A)\ge 1-\delta, \text{ for all } A\in\mathcal A.
\end{equation}

In this section, we focus for simplicity on {\em one-sided prediction intervals}; by applying our results twice, or to a non-conformity score, we achieve two-sided intervals (see Corollary~\ref{rem:twice} and Remark~\ref{rem:score} below). Moreover, as usual in the multicalibration literature, our results hold even for the case of arbitrarily intersecting population subgroups.  A somewhat different version of the intersectional case was also studied in~\cite{gupta2021online, jung2022batch}; later, we will explain how~\HM{} can be used for this as well.

In the one-sided version of \eqref{def:equalized coverage: orginal}, we let $\Cint(x)=(l_{\delta}(x),\infty)$ be a one-sided $(1-\delta)$-prediction interval, and study the following coverage criterion:
\begin{equation}\label{def:equalized coverage: ours}
\Prob(x\in A)\cdot |\Prob(y\ge l_{\delta}(x)\mid x\in A)- (1-\delta)|\le \alpha, \text{ for all } A\in\mathcal A,
\end{equation}
which can then be rewritten as $$
\left|\E[\1(x\in A)\cdot\left(\1(y\ge l_{\delta}(x))- (1-\delta)\right)]\right|\le\alpha, \text{ for all } A\in\mathcal A.
$$
{ 
A few remarks about the problem definition are in order.
First, there is a trivial deterministic solution to the original equalized coverage problem in Equation~\eqref{def:equalized coverage: orginal}, whether or not the groups are disjoint: just set
$\Cint(x)=\R$.\footnote{This does not mean that previous algorithms are trivial!}
Similarly, while our (one- or two-sided) version in Equation~\eqref{def:equalized coverage: ours} rules this out, it admits the trivial randomized solution in which we take  $l_\delta(x)=-\infty$ with probability $1-\delta$, and $l_\delta(x)=\infty$ with probability $\delta$. 
One approach to ruling out both trivial solutions is to require a stronger condition, {\em multivalidity}, proposed in~\cite{gupta2021online,jung2022batch}.
Adopting our notation, instead of asking for a small $|\Prob(y\ge l_{\delta}(x)\mid x\in A)- (1-\delta)|$ as in  Eq.~\eqref{eq:conformal:extend}, \textit{multivalidity} requires $|\Prob(y\ge l_{\delta}(x)\mid x\in A, l_{\delta}(x))- (1-\delta)|$ to be small. Analogous to the relationship between multicalibration and multi-accuracy, 
multivalidity is stronger than our requirement in Equation~\eqref{eq:conformal:extend}.  As discussed below, we can use \HM{} for this as well.  In
Theorem~\ref{thm:no harm}, we will further provide a different argument that our algorithm is doing something nontrivial, even when we do not enforce multivalidity. 
}

In the following, we generalize Eq.~\eqref{def:equalized coverage: ours} to obtain the {\em Intersectional Equalized Coverage} requirement:
\begin{equation}\label{eq:conformal:extend}
        \sup_{c\in\mathcal C}\Big|\E[ c(x)(\1\{y\ge l_{\delta}(x)\}-(1-\delta))]\Big|\le\alpha,
\end{equation}
where $\mathcal C$ denotes an arbitrary pre-specified collection of functions (including indicator functions of pre-specified sub-populations, and also more general continuous functions, which is typical in the multicalibration literature).  For example, $\cC$ might be the functions that can be computed by decision trees of a fixed depth. 
Applying \HM{} with $s(l,y)=(1-\delta)-\1\{l\le y\}$, yields the following result.

\begin{theorem}\label{thm:conformal} 
Suppose that $y\mid x$ is a continuous random variable \footnote{{Our analysis can be directly extended to the case where $y$ is stored in finite precision and $\alpha$ is taken to be larger than the precision error}}, the conditional density of $y$ given $x$ is upper bounded by $K_p$, and $|\E[y]|<C$ for some universal constant $C>0$. In addition, suppose that $\bE_{x\sim \cD^\cX}[c^2(x)]\le B$ for all $c\in\mathcal C$.
Then for a suitably chosen $\eta=\cO(\alpha/(K_pB))$, using the potential function $\cL(l,y)=(1-\delta)\cdot l-\min(l-y,0)$, \HM{} (Algorithm~\ref{alg:gmc}) converges in $T=\cO(K_p B^2/\alpha^2)$ steps, and outputs a function $l_\delta(\cdot)$ satisfying
$$
        \sup_{c\in\mathcal C}\Big|\E[ c(x)(\1\{y\ge l_{\delta}(x)\}-(1-\delta))]\Big|\le\alpha.
$$
\end{theorem}
\begin{proof}[Proof of Theorem~\ref{thm:conformal}]
In order to apply Theorem~\ref{thm:gmc}, it is sufficient to verify that the potential function $\cL(l,y)=(1-\delta)l-\min(l-y,0)$ - the smooth-like condition, that is
 $$
 \E[\cL(l,y)-\cL(l',y)]\ge \E[(l(x)-{l'}(x))s(f(x),y)] - \kappa_L  \E[(l(x)-l'(x))^2].
 $$

In particular, we  have \begin{align*}
    \cL(l,y)-\cL(l',y)=&(1-\delta)(l-l')+\min(l'-y,0)-\min(l-y,0)\\
    =&(1-\delta)(l-l')+(l'-l)\1\{l-y,l'-y<0\}-(l-y)\1\{1'-y>0>l-y\}\\
        &+(l'-y)\1\{1'-y<0<l-y\}\\
    =&(1-\delta)(l-l')+(l'-l)\1\{l-y<0\}-(l'-y)\1\{1'-y>0>l-y\}\\
    &+(l'-y)\1\{1'-y<0<l-y\}\\
    =&(l-l')((1-\delta)-\1\{l-y<0\})+(l'-y)(\1\{1'<y<l\}-\1\{1'>y>l\}).
\end{align*}
We have 
\begin{align*}
|(l'-y)(\1\{1'<y<l\}-\1\{1'>y>l\})|\le&  |l-l'|\cdot(\1\{1'<y<l\}+\1\{1'>y>l\})).
\end{align*}
The above inequality is due to the fact that if a $|l-l'|$ perturbation can change the sign of $l'-y$, then $|l'-y|<\eta|c(x)|$.

Assuming the conditional density of $y$ given $x$ is upper bounded by $K_p$, we then have $$
\E[|(l'-y)(\1\{\ell'<y<l\}+\1\{\ell'>y>l\})|]\le L \eta^2 \E[c(x)^2]\le K_p |l-l'|^2.
$$

In addition, we verify that $\E[\cL(l,y)]$ has a uniform lower bound. We discuss in cases. If $l-y<0$ (i.e., $l<y$), we have 
$\cL(l,y)=(1-\delta)\cdot l-\min(l-y,0)=(1-\delta)\cdot l-l+y=y-\delta \cdot l>(1-\delta)\cdot y$; if $l-y>0$ (i.e., $l>y$), we have 
$\cL(l,y)=(1-\delta)\cdot l-\min(l-y,0)=(1-\delta)\cdot l>(1-\delta)\cdot y.$

Since $|\E[y]|<C$ for some universal constant $C>0$ by assumption, we have $$
\E[\cL(l,y)]\ge -(1-\delta)C.
$$
\end{proof}

The following Corollary is immediate from Theorem~\ref{thm:conformal}. 
\begin{corollary}
\label{rem:twice}
If we apply the algorithm above (i.e. the HappyMap specialized to our current setting) to two different cutoff values $\delta/2$ and $(1-\delta/2)$ and obtain $l_{\delta/2}(x)$ and $l_{1-\delta/2}(x)$ respectively, we obtain the two-sided prediction interval $\Cint(x)=[l_{\delta/2}(x),l_{1-\delta/2}(x)]$ such that $\sup_{c\in\mathcal C}\E[ c(x)(\1\{y\in\Cint(x)\}-(1-\delta))]\le2\alpha
$.
\end{corollary}
{
\begin{remark}\label{rem:score}
In the conformal prediction literature, there is another commonly used way to construct two-sided prediction intervals based on non-conformity scores. More specifically, the  non-conformity score $m(x,y)$ is a metric that measures how the response $y$  fails to ``conform'' to a prediction $h(x)$, where $h$ is an arbitrarily fixed prediction function. For example, a popular choice of the  non-conformity score for regression tasks is $m(x,y)=|y-h(x)|$. Additional choices for the non-conformity score can be found in \cite{shafer2008tutorial, angelopoulos2021gentle}. Given a non-conformity score $m(x,y)$, the set $\{y: m(x,y)\le f(x)\}$ will naturally yield a two-sided interval for $y$. To this end, applying our method directly to $(x,\tilde y)$ with $\tilde y:=m(x,y)$ will then produce a valid $(1-\delta)$ prediction interval.
\end{remark}
}

Recall that multivalidity \cite{gupta2021online, jung2022batch}  bounds
$|\Prob(y\ge l_{\delta}(x)\mid x\in A, l_{\delta}(x))- (1-\delta)|$ for all $A$ and value of $l_\delta(x)$. Analogous to the relationship between multicalibration and multi-accuracy, 
multivalidity is generally stronger than the requirement in Equation~\eqref{eq:conformal:extend}, {resulting in potentially much longer prediction intervals  (a more detailed discussion is deferred to Section~\ref{sec:comp}). } 
By considering $c$ of the form $c(l_{\delta}(x),x)=\1\{l_{\delta}(x)\in I\}$ for $I\in\mathcal I$, where $\mathcal I$ is the collection of small bins $\mathcal I=\{[-C,-C+\lambda], [-C+\lambda,-C+2\lambda],...,[C-2\lambda, C-\lambda] ,[C-\lambda, C]\}$ for some discretization level $\lambda$ and $C$ being the upper bound of $|y|$, 
  \GMC{} recovers multivalidity, and applying Algorithm~\ref{alg:gmc} yields a new algorithm that achieves multivalidity.

  {The extension of \eqref{def:equalized coverage: orginal} to intersecting sets is also considered in \cite{foygel2021limits}.
  They define the set collection $\cA$ {\it ex post facto}, after the training data have been collected, to be all sufficiently large subsets of the training set, and then enumerate over all these (exponentially many) sets.
  In contrast, as is typical in the multicalibration literature, we name the sets {\it a priori} and rely on weak learning, which can be more efficient than exhaustive search when the collection $\mathcal A$ has special structure, even if the collection is infinite.
  } 

{HappyMap produces a prediction interval that is fair with respect to a large collection of groups. In the following, we present a result analyzing the utility, {\it i.e.}, the width of the constructed prediction interval, when \HM{} is applied as {\em post-processing} of a high-quality, but not necessarily group-fair, initial conformal map~$l_0$.  We are agnostic regarding the source of $l_0$: it may be given to us or we may obtain it by splitting the sample and building a high-quality but fairness-unaware conformal mapping using any standard method.}

To facilitate the analysis, we first introduce the following two definitions on quantiles. These definitions are standard and follow the literature of quantile estimation, see \cite{christmann2007svms, steinwart2011estimating} and the reference therein.
\begin{definition}\label{def:quantile}
For a distribution $Q$ with supp $Q \subset [-C, C]$ for a universal constant $C>0$, let us denote the $\tau$-quantile of $Q$ by $F_\tau(Q):=\{t\in\R:Q((-\infty,t])\ge \tau \text{ and } Q([t,\infty)])\ge 1-\tau\}$. In case of $F_\tau(Q)$ being an interval, we write $t_{\min}(Q)=\min F_\tau(Q)$ 
and $t_{\max}(Q)=\max F_\tau(Q)$. $Q$ is said to have a $\tau$-quantile of type $q \in (1, \infty)$ if there exist constants $\alpha_Q >0$ and $b_Q > 0$ such that for all $s\in[0,\alpha_Q]$, \begin{align*}
    Q(t_{\min}-s,t_{\min})&\ge b_Qs^{q-1}\\
        Q(t_{\max},t_{\max}+s)&\ge b_Qs^{q-1}.
\end{align*}
\end{definition}
Since we are not interested in a single distribution $Q$ on $\R$ but in distributions $P$ on $\mathcal X \times \R$,  the following definition extends the previous definition to such $P$. 

\begin{definition}
Let $p \in (0, \infty]$, $q \in (1, \infty)$
and $P$ be a distribution on $\mathcal X \times \R$.
$P$ is said to have a $\tau$-quantile of $p$-average type $q$, if there exists a set $\Omega_{\mathcal X}\subset \mathcal X$ such that $\Prob(x\in\Omega_{\mathcal X})=1$, 
 $supp( P(\cdot|x)) \subset [-C,C]$  for  a universal constant $C>0$, $P(\cdot|x)$ has a $\tau$-quantile of type $q$ for all any $x \in \Omega_{\mathcal X}$, and there exist $\alpha_{P(\cdot\mid x)}$ and $b_{P(\cdot\mid x)}$ as defined in Definition~\ref{def:quantile}, such that $\E[|b_{P(\cdot\mid x)}^{-p}(x)\cdot \alpha_{P(\cdot\mid x)}^{(1-q)p}(x)|]<\infty$.
\end{definition}

We then have the following result showing that applying HappyMap comes at nearly no cost in the width of the prediction interval: if the input $l_0$ is close to the optimal quantile function, then the HappyMap algorithm can preserve this approximation.
\begin{theorem}\label{thm:no harm}
Suppose the conditions in Theorem~\ref{thm:conformal} hold, and $P$, the distribution on $\cX \times \cY$, has a $\delta$-quantile of $p$-average type $q$. Assume $pq/(p+1)\ge 2$. Let $l_{\delta}^*(x)\in F_\delta(P(\cdot\mid x))$ be any element in the $\delta$-quantile of $Y\mid X=x$. If an input $l_0$ satisfies $\E_{x\sim P_X}(l_0(x)-l_{\delta}^*(x))^2\le \beta$, then there exists a universal constant $C$, such that the output of the HappyMap Algorithm~\ref{alg:gmc}, $l_\delta$, satisfies $$
\E(l_\delta(x)-l_{\delta}^*(x))^2\le C\beta.
$$
\end{theorem}
\begin{proof}[Proof of Theorem~\ref{thm:no harm}]
Recall that we use the potential function $\cL(l,y)=(1-\delta)l-\min(l-y,0)$ when we apply HappyMap to this current problem. According to the derivations in the Proof of Theorem~\ref{thm:conformal}. We have that \begin{align*}
    &|\E[\cL(l,y)]-\E[\cL(l',y)]|\\
    &\le |\E[(l-l')((1-\delta)-\1\{l-y<0\})]|+|\E[(l'-y)(\1\{1'<y<l\}-\1\{1'>y>l\})]|\\
    &\le|\E[(l-l')((1-\delta)-\1\{l-y<0\})]|+K_p\cdot\E[|l-l'|^2].
\end{align*}
Since $l_\delta^*(x)=\argmin_l \E[\cL(l,y)\mid x]$, and $\Prob(l_\delta^*(x)<y\mid X=x)=1-\delta$, we have
\begin{align*}
    \E[\cL(l_0,y)]-\E[\cL(l_\delta^*,y)]
    \le K_p\cdot \E|l_0(x)-l_\delta^*(x)|^2\le K_p \beta.
\end{align*}
Moreover, by the Proof of Theorem~\ref{thm:conformal}, we have $\E[\cL(l,y)]\le \E[\cL(l_0,y)]$, implying 
\begin{align*}
    \E[\cL(l,y)]-\E[\cL(l_\delta^*,y)]
    \le K_p \beta.
\end{align*}
Then by Theorem 2.7 in \cite{steinwart2011estimating}, and use the fact that the function $L_2$ norm is dominated by the function $L_{pq/(p+1)}$ norm when $pq/(p+1)\ge 2$, we have $$
\E[|l(x)-l^*_\delta(x)|^2]\le C' \cdot\E[\cL(l,y)]-\E[\cL(l_\delta^*,y)] \le C' K_p \beta.
$$
\end{proof}
The above theorem indicates that when the conditional distribution $Y\mid X=x$ is unimodal and symmetric\footnote{similar assumptions have been used in conformal prediction literature, see, eg. \cite{lei2014distribution,lei2021conformal}}, if we have a good initialization function $f_0$ with approximation error $\beta=o(1)$, then the prediction interval constructed as in either Corollary~\ref{rem:twice} or Remark~\ref{rem:score} would converge to the width of the optimal $(1-\delta)$ prediction interval $[l^*_{\delta/2}(x), l^*_{1-\delta/2}(x)]$ where we know the distribution $Y\mid X=x$ exactly. We note that $\beta=o(1)$ can be potentially achieved by first splitting the data into two halves: $l_0$ is trained on the first half, and processed by HappyMap on the second half. Such data-splitting is common in the conformal prediction literature~\cite{lei2014distribution,romano2019conformalized, lei2021conformal}.

\section{Application: Target-independent Learning}\label{sec:ua}
Kim {\it et al.}~\cite{kim2022universal} demonstrated that, with an appropriate collection $\cC$ of propensity reweighting functions, training a $(\cC,\alpha)$-multicalibrated predictor on source data $\cD_{so}$ yields efficient estimation of statistics of interest on data drawn from previously unseen target distributions $\cD_{ta}$~\cite{kim2022universal}.
In this section, we use \HM{} to extend these results to problems in target-independent learning that lie beyond the statistical estimation problems considered in \cite{kim2022universal}:
target-independent prediction and uncertainty quantification. Our approach also yields a fruitful new perspective on analyzing missing data, giving new solutions to this problem.

Let us use $\cZ=\{so,ta\}$ to indicate sampling in the source and target distributions, respectively. As in \cite{kim2022universal}, we consider a joint distribution $\cD_{joint}$ over $(x,y,z)$ triples. The source, respectively, target, populations $\cD_{so}$ and $\cD_{ta}$ can be viewed as the joint distribution on $(x,y)$, conditioning on $z=ta$ or $z=so$, respectively. We similarly use the notation $D^\cX_z$ and $D^\cY_z$ to denote the marginal distributions obtained by conditioning on $z\in\cZ$. As in~\cite{kim2022universal}, we assume certain functional relationship between $x$ and $y$ is the same on the distributions $\cD_{so}$ and $\cD_{ta}$, which is known as ``ignorability" in the causal inference literature and "covariate shift" in machine learning.  Formally, we have:

\begin{assumption}[Covariate shift assumption] \label{ass:cov}
For a triplet $(x,y,z)$ drawn from $\cD_{joint}$,
\begin{small}
$$\bP_{(x,y,z)\sim \cD_{joint}}(x,y,z)=\bP(x)\bP(y|x)\bP(z|x).$$
\end{small}
\end{assumption}
\noindent
Based on Assumption~\ref{ass:cov}, we have
\begin{small}
$$\bP_{(x,y)\sim \cD_{so}}(y|x)=\bP_{(x,y)\sim \cD_{ta}}(y|x).$$
\end{small}
By convention and without loss of generality, throughout this section, we assume uniform prior over $\cZ=\{so,ta\}$, i.e. $\bP(z=ta)=\bP(z=so)$.

A common tool for performing covariate shift studies is {\em Propensity score reweighting}. The propensity score is defined to be  $e(x)=\bP(z=so|x)$, and the propensity score ratio
$\frac{1-e(x)}{e(x)}=\frac{\bP(z=ta|x)}{\bP(z=so|x)}$ 
can be used to convert an expectation over $\cD^\cX_{so}$ to an expectation over $\cD^\cX_{ta}$.
However, without observing samples from $\cD_{ta}$ at training time, we cannot estimate the propensity score ratio. Kim {\it et al.}~\cite{kim2022universal} proposed multicalibrating with respect to the class $\cC=\{\frac{1-\sigma(x)}{\sigma(x)}:\sigma\in\Sigma\}$, for a family of functions $\Sigma$ that (it is hoped) captures the propensity score ratios of interest.  They showed that multicalibration with respect to this class (in fact, even multi-accuracy with respect to this class) ensures that the resulting predictor provides estimation accuracy competitive with the best propensity reweighting function in the class~$\cC$, a notion they call {\em universal adaptability}.
In the realizable, case, when the propensity ratio for the unseen target domain is in the class $\cC$, $f$ is guaranteed to yield a good estimate for the statistic of interest on the target domain.

\subsection{Universally Adaptive Predictors under $\ell_2$ Loss}\label{subsec:acctransfer}
As a warm-up exercise, which does not require the full power of \HM , we consider universal adaptivity of {\em predictors} $f:\cX\rightarrow [0,1]$ under $\ell_2$ loss.  This is more complex than statistical estimation under covariate shift, and it requires a more complicated class of functions $c(f(x),x)$. 

Formally, our goal is to obtain a prediction $f$ with a small estimation error 
$$\bE_{(x,y)\sim \cD_{ta}}(f(x)-\bE_{(x,y)\sim \cD_{ta}}[y|x])^2$$
for $\cY=[0,1]$. We note that this quantity is commonly used as a measure of the quality of a prediction function.  
By Assumption~\ref{ass:cov} (covariate shift), $\bE_{(x,y)\sim \cD_{so}}[y|x]=\bE_{(x,y)\sim \cD_{ta}}[y|x])$, yielding
\begin{align*}
 \bE_{(x,y)\sim \cD_{ta}}(f(x)-\bE_{(x,y)\sim \cD_{ta}}[y|x])^2  &= \bE_{(x,y)\sim \cD_{ta}}(f(x)-\bE_{(x,y)\sim \cD_{ta}}[y|x])(f(x)-y)\\
 &=\bE_{(x,y)\sim \cD_{ta}}(f(x)-\bE_{(x,y)\sim \cD_{so}}[y|x])(f(x)-y).
\end{align*}
Using the propensity score ratio, we have 
$$\bE_{(x,y)\sim \cD_{ta}}(f(x)-\bE_{(x,y)\sim \cD_{so}}[y|x])(f(x)-y)=\bE_{(x,y)\sim \cD_{so}} \left( \frac{1-e(x)}{e(x)}(f(x)-\bE_{(x,y)\sim \cD_{so}}[y|x]) \right) (f(x)-y).$$
At training time, we cannot see the samples from $\cD_{ta}$, so we cannot estimate $e(\cdot)$ from samples as in the classical reweighting approach.  However, following \cite{kim2022universal}, one can use a class $\{\frac{1-\sigma(x)}{\sigma(x)}:\sigma\in\Sigma\}$ to represent the propensity score ratios of interest, and try to find an $f:\bR\mapsto [0,1]$ with a small error:
$$\sup_{\sigma\in\Sigma}\bE_{(x,y)\sim \cD_{so}} \left( \frac{1-\sigma(x)}{\sigma(x)}(f(x)-\bE_{(x,y)\sim \cD_{so}}[y|x]) \right)\cdot(f(x)-y).$$
This almost gives us the form we need in Eq.~\eqref{eq:1}; it remains only to deal with  $\bE_{(x,y)\sim \cD_{so}}[y|x]$.  This is accomplished by introducing a class $\cP=\{p:\cX\mapsto [0,1]\}$ containing, it is hoped, a good approximation of $\bE_{(x,y)\sim \cD_{so}}[y|x]$, and trying to find an $f:\bR\mapsto [0,1]$ such that
\begin{equation}\label{eq:acctransfer}
    \Big|\sup_{\sigma\in\Sigma,p\in\cP}\bE_{(x,y)\sim \cD_{so}} \left( \frac{1-\sigma(x)}{\sigma(x)}(f(x)-p(x)) \right) \cdot (f(x)-y)\Big|\le \alpha,
\end{equation}
for some small value~$\alpha$. In other words, we take $\cC=\{c(f(x),x)=\pm\frac{1-\sigma(x)}{\sigma(x)}(f(x)-p(x)):\sigma\in\Sigma,p\in\cP\}$, where $\sigma,p\in(0,1)$. We define the approximation error by
$$\beta_1(p)=\sqrt{\bE_{(x,y)\sim\cD_{so}}(p(x)-\bE_{(x,y)\sim\cD_ {so}}[y|x])^4}$$ and $$\beta_2(\sigma)=\sqrt{\bE_{(x,y)\sim\cD_{so}}(\frac{1-\sigma(x)}{\sigma(x)}-\frac{1-e(x)}{e(x)})^4}.$$

Although we do not use the full power of \HM , since we have stayed with the original mapping $(f(x)-y)$, we can nonetheless apply Theorem~\ref{thm:gmc} to obtain:
\begin{theorem}\label{thm:acctransfer}
 Assume $\sigma(\cdot)\in(c_1,c_2)$, where $0<c_1<c_2<1$. Then, we have  $|c(f(x),x)|\le 2(1-c_1)/c_1:=B$. Let   $\beta=\inf_{\sigma\in\Sigma,p\in\cP}\sqrt{2B^2\beta_1(p)+2\beta_2(\sigma)}.$ Suppose we run Algorithm~\ref{alg:gmc} with a suitably chosen $\eta=\cO(\alpha/B)$, then the algorithm converges in $T=\cO(2B/\alpha^2)$ iterations, using the potential function $\cL(f(x),y)=1/2(f(x)-y)^2$ and $O=[0,1]$,  which results in $$\Big|\bE_{(x,y)\sim \cD_{so}}[\frac{1-e(x)}{e(x)}(f_t(x)-\bE_{(x,y)\sim\cD_{so}}[y|x])(f_t(x)-y)]\Big|\le\alpha+\beta,$$
for the output $f_t(\cdot)$ of Algorithm~\ref{alg:gmc}.
\end{theorem}
The proof of Theorem~\ref{thm:acctransfer} is deferred to Section~\ref{sec:app:acctransfer}.

\subsection{Universally Adaptive Conformal Prediction }\label{subsec:conformal}

 In this section, we apply \HM{} to achieve universally adaptive conformal prediction, wherein we train a conformal prediction model on source data while ensuring validity on unseen target distributions, in the covariate shift setting. 
 
Recall that the goal of standard conformal prediction is to obtain a $(1-\delta)$ prediction interval for $y$. 
 In this section, for simplicity of presentation, we consider the one-sided interval $(l(x), \infty)$ and focus on the case in which the response $y$ is continuous\footnote{As in Section~\ref{sec:application}, we can handle the two-sided case with two applications of our technique (Corollary~\ref{rem:twice}).} Letting $\cD_{ta}$ denote the unseen target distribution, our goal is to construct $l(\cdot)$ such that
    $\Prob_{(x,y)\sim \cD_{ta}}(l(x)\le y)$ is close to the desired level $(1-\delta)$: $$
    |\Prob_{(x,y)\sim \cD_{ta}}(l(x)\le y)-(1-\delta)|\le\alpha.
    $$

 Note that the above inequality implies $\Prob_{(x,y)\sim \cD_{ta}}(l(x)\le y)\ge 1-\delta-\alpha$, which is the more standard requirement used in the conformal prediction literature.  By Assumption~\ref{ass:cov}, one can rewrite this probability: 
    $$
   \Prob_{(x,y)\sim \cD_{ta}}(l(x)\le y)= \bE_{(x,y)\sim \cD_{so}}\Big(\frac{1-e(x)}{e(x)}\1\{l(x)\le y\}\Big).
    $$
Now, in the same spirit as in the previous section (specifically, Eq.~\eqref{eq:acctransfer}), we seek
$l(\cdot)$ such that
    \begin{equation}\label{eq:conformal}
        \Big|\sup_{\sigma\in\Sigma}\bE_{(x,y)\sim \cD_{so}}[\frac{1-\sigma(x)}{\sigma(x)}\cdot(\1\{l(x)\le y\}-(1-\delta))]\Big|\le \alpha,
\end{equation}
for unseen $\cD_{ta}$ and for some $\alpha>0$.  To this end, we apply \HM{} (Algorithm \ref{alg:gmc}) with $\cC=\{c(x)=\frac{1-\sigma(x)}{\sigma(x)}: \sigma\in\Sigma\}$ and the mapping $s(l,y)=\1\{l\le y\}-(1-\delta)$.

\begin{theorem}\label{thm:one-sided}
Suppose that $y\mid x$ is a continuous random variable, the conditional density of $y$ given $x$ is upper bounded by $K_p$, and $|\E[y]|<C$ for some universal constant $C>0$. Assume $\E_{x}[c^2(x)]\le B$,  $\forall c \in \cC$, and consider the  \HM{} meta-algorithm (Algorithm~\ref{alg:gmc}) with $s(l,y)=\1\{l\le y\}-(1-\delta)$. 
Then for any target distribution whose marginal density function of $x$, $p_{ta}(x)$, and let $\beta$ satisfy {$\inf_{c\in\mathcal C}{\E[(c(x)-\frac{p_{ta}(x)}{p_{so}(x)})^2]}\le\beta^2,$}
there exists a choice of $\eta=\cO(\alpha/(K_pB))$ such that Algorithm~\ref{alg:gmc} terminates in $T=\cO(LB/\alpha^2)$, using the potential function $\cL(l,y)=(1-\delta)\cdot l-\min(l-y,0)$. The resulting function $l_{\delta}$ satisfies 
    $$
    | \Prob_{(x,y)\sim \cD_{ta}}(l_{\delta}(x)\le y)-(1-\delta)|\le\alpha+\beta.
    $$
\end{theorem}
The proof Theorem~\ref{thm:one-sided} follows a combination of the universal adaptivity argument and the analysis of  Theorem~\ref{thm:conformal}, and is deferred to Section~\ref{sec:one-sided}.

    \begin{remark} {To our knowledge, Tibshirani {\it et al.}~\cite{tibshirani2019conformal} were the first to consider conformal prediction under covariate shift. Their method relies on exact knowledge of $(1-e(x))/e(x)$. In contrast, we achieve a valid asymptotic coverage guarantee 
    without exact knowledge of $(1-e(x))/e(x)$, provided something close to this score is in~$\Sigma$.}
    \end{remark}
 
 Finally, we note that we can extend our results to universally adaptive equalized coverage and universal multivalidity if we enrich $\mathcal C$. For example, by including in $\mathcal C$ both $\{c(x)=\frac{1-\sigma(x)}{\sigma(x)}: \sigma\in\Sigma\}$ and the sub-population indicator functions, we can achieve equalized coverage under covariate shift.

\subsection{Prediction with Missing Data}
\GMC{} can also be applied to the study of missing data, a common issue in statistical data analysis \cite{little2019statistical,tony2019high}. Suppose we consider learning tasks that predict response $y\in [0,1]$ based on features $x \in\mathcal X$.  Let $X$ denote the complete data matrix, where the $i$-th row $x_i$ is the features of the $i$-th observation and the response vector is $Y=(y_1,...,y_n)^\top$. Suppose the data are i.i.d., with $(x_i, y_i)\stackrel{i.i.d.}{\sim} \cD$. In the missing data problem, some entries of $X$ are missing, so we define $x^{(1)}_{i}$ to be the components of $x_i$ that are observed for observation $i$, and $x^{(0)}_{i}$ the components of $x_i$ that are missing.  In addition, we let $R_i$ be the indicator of complete cases, that is, $R_i=1$ iff $x_i$ is fully observed (i.e. $x^{(1)}_i=x_{i}$).

There are three major missing data mechanisms: missing completely at random (MCAR),
missing at random (MAR) and missing not at random (MNAR). We say a dataset is MCAR if the distribution of missingness indicators $R$ is independent of $X$. For MAR,  $R$ depends on $X$ only
through its observed components, i.e., $R_i \indep x_i^{(0)}|x_i^{(1)}$. 
For MNAR, the distribution of $R$ could depend on
the missing components of $X$. Our method, described next, applies to all three mechanisms.

Suppose our goal is to learn a predictor function $f$ such that the test mean
squared error (MSE) $\E_{(x,y)\sim \cD}[(y-f(x))^2]$ is small. There are two principal methods in the literature: weighting and imputation.  The weighting approach~\cite{li2013weighting} minimizes the following loss function that is evaluated on the complete data, 
\begin{equation}\label{eq:missingdata}
    \arg\min_{f} \sum_{i=1}^n {\1\{R_i=1\}}{w(x_i)} (y_i-f(x_i))^2,
\end{equation}
where $w(x_i)$ approximates $\frac{\bP(x_i\mid R_i=1)}{\bP(x_i)}$.

We will adapt the weighting approach to our framework, and use $\mathcal C$ as a class of functions where some $c(f(x),x)\in\mathcal C$ approximates $\frac{\bP(x\mid R_i=1)}{\bP(x)}(\E[y| x]-f(x))$. 
Since $\E_{(x,y)\sim \cD}[(y-f(x))^2]$ can be equivalently written as $\E_{(x,y)\sim \cD\mid R=1}[\frac{\bP(x\mid R_i=1)}{\bP(x)}(\E[y| x]-f(x))\cdot(y-f(x))]$. Instead of minimizing the loss function as in \eqref{eq:missingdata}, we aim to find an $f$ such that  $\sup_{c\in\mathcal C}\E_{(x,y)\sim \cD\mid R=1}[c(x)(y-f(x))]\le\alpha$, where $\mathcal C$ is a class of functions that approximates $\frac{\bP(x\mid R_i=1)}{\bP(x)}(\E[y| x]-f(x))$. Applying \HM{} with such a function class $\cC$ and $s(f,y)=f-y$, we obtain the following result. 
\begin{theorem}
\label{thm:missing}
Suppose $\inf_{c\in\mathcal C}\sqrt{\E[( c(f(x),x)-\frac{\bP(x\mid R_i=1)}{\bP(x)}(\E[y| x]-f(x))]^2}\le\beta$, and $$\sup_{c\in\mathcal C}\E[c^2(f(x),x)]\le B.$$ Consider the  \HM{} meta-algorithm (Algorithm~\ref{alg:gmc}) with $s(f,y)=f-y$. 
Then there exists a suitable choice of $\eta=\cO(\alpha/B)$ such that Algorithm~\ref{alg:gmc} terminates in $T=\cO(LB/\alpha^2)$, using the potential function $\cL(f(x),y)=1/2(f(x)-y)^2$ and $O=[0,1]$. The resulting function $l_{\delta}$ satisfies
$$
\E_{(x,y)\sim \cD}[(y-f(x))^2]\le\alpha+\beta.
$$
\end{theorem}
The proof of Theorem~\ref{thm:missing} is deferred to Section~\ref{sec:app:missing}.
\begin{remark} 
The analysis above considered learning a predictor with small mean squared error in the missing data setting. 
The same idea can be used to other settings when there is missing data, including constructing prediction intervals when part of the observed data are missing, and enforcing multi-group fairness notions (that are studied in Section~\ref{sec:application}) with missing data.
\end{remark}

\section{Discussion and Conclusion}
In this paper, we propose \GMC, a generalization of
multi-calibration, unifying many existing algorithmic fairness notions and target-independent learning
approaches, and yielding a wide range of new applications, including a new fairness notion
for uncertainty quantification, a novel technique for conformal prediction under covariate shift,
and a different approach to analyzing missing data.  At a higher level, we advance the field of macro-learning.

An interesting future direction is to extend \GMC{} to (low-dimensional) representation functions. As argued recently in the deep learning community, pre-training or representation learning is a powerful tool in improving final prediction accuracy\footnote{For example, in Table~5 of {https://arxiv.org/pdf/2204.02937.pdf}, pre-training improves  accuracy increases from 60\% to 90+\%.}. 
The hope is that a suitable extension of \GMC{} for representation learning will further facilitate target-independent learning by using an extended \HM{} that post-processes the representation learned through pre-training. 
In addition, our current target-independent learning requires the covariate shift assumption (that is, the conditional distribution $y\mid x$ is invariant across different environments). An interesting question is to study how to use \GMC{} for target-independent learning when there are some other types of invariance, {\it e.g.}, as in invariant risk minimization (IRM)~\cite{arjovsky2019invariant,deng2020representation,deng2021adversarial}.

The power of the ``multi-X"/macro-learning framework, including \GMC , stems from choosing a rich collection $\mathcal C$. However, a rich $\mathcal C$ means a harder task for the weak agnostic learner, and increased sample complexity.  This motivates our investigation of reduced function class complexity (Section~\ref{sec:reduction}).  
To fully realize the potential of macro-learning it would be helpful to understand, even from the perspective of heuristics, the trade-off between the power of $\cC$ and the computational complexity of the weak learner.

\bibliography{cite,cite2}

\appendix

\section{Additional Literature Review}\label{app:liter}
  H\'ebert-Johnson {\it et al.}\,introduced the method for obtaining efficient multi-calibrated predictors in the batch setting~\cite{hebert2018multicalibration}, multicalibration in the online setting has a long history in statistics (not with this name, however; see, for example, \cite{FosterVohra,G3} and the references therein). {In parallel with \cite{hebert2018multicalibration}, \cite{kearns2018preventing} introduced notions of multi-group fairness, including statistical parity subgroup fairness and false positive subgroup fairness.}   \cite{jung2021moment} extends the notion of multicalibration to higher moments in the batch setting, and~\cite{gupta2021online} provides an online solution. {In particular, taking $s(f(x),y)=f(x)^k-y^k$ recovers the raw moment version of moment multicalibration; they also consider calibrating high-order central moments  using a novel and different technique.} Moreover, multicalibration has also been applied in the batch setting to solve several problems of the same flavor: fair ranking~\cite{DKRRY2019}; {\em ommiprediction}, which is roughly learning a predictor that, for a given class of loss functions, can be post-processed to minimize loss for any function in the class~\cite{gopalan2021omnipredictors}; and providing an alternative to propensity scoring for the purposes of generalization to future populations~\cite{kim2022universal}. We summarize those previous concepts in the following table and show their relationship with our new notion.

\begin{table}[h!]
\centering
\begin{tabular}[c]{ |p{5cm}||p{5.5cm}|p{3cm}| }

\hline
 \multicolumn{3}{|c|}{Relationship between previous concpets and the $s$-Happy Multicalibration} \\
 \hline
Concept name &choice of $c(f(x),x)$&choice of $s(f(x),y)$\\
 \hline
 Multiaccuracy \cite{kim2019multiaccuracy}   & $c(x)$    &$f(x)-y$\\

 Multicalibration \cite{hebert2018multicalibration}&   $c(f(x),x)$  & $f(x)-y$   \\
 Low-degree Multicalibration \cite{gopalan2022low}   &$\tilde{c}(x)w(f(x))$ for some function $\tilde c$ & $f(x)-y$\\
 Minimax Group Fairness \cite{diana2021minimax}&   $\1\{x\in G\}$ for demographic group $G$  & non-negative loss\\
Omni-predictor \cite{gopalan2021omnipredictors} (more in Appendix~\ref{sec:exist})& $\1\{f(x)\in I_v\}$ and $c(x)\1\{f(x)\in I_v\}$   & $f(x)-y$   \\
 \hline
\end{tabular}
\caption{Summary of the relationship with prevoius concepts}\label{tab}
\end{table}

\section{Omitted Details for Section~\ref{sec:notion}}\label{app:notion}

\subsection{Details for the Sample Version Algorithm}\label{app:sample version}

In this section, we will provide the corresponding sample version algorithm for Algorithm~\ref{alg:gmc} and its formal theoretical guarantee. Before that, similarly as in \cite{kim2019multiaccuracy}, let us introduce an \textbf{informal} but handy concept. For a dataset $D$, we use $\bE_{(x,y)\sim D}$ to denote the empirical expectation over $D$. Throughout this section, we assume $\sup_{c\in\cC}\|c\|_\infty\le B_0$ for some universal constant $B_0>0$.

\begin{definition}[Dimension of a function class]
We use $d(\cC)$ to denote the dimension of an agnostically learnable class $\cC$, such that if the sample size $m\ge C_1 \frac{d(\cC)+\log(1/\delta)}{\alpha^2}$ for some universal constant $C_1>0$, then the random samples $S_m$ from $\cD$ guarantee uniform convergence over $\cC$ with error at most $\alpha$ with failure probability at most $\delta$, that is, for any fixed $f$ and fixed $s$ with $\|s\|_\infty\le C_2$ for some universal constant $C_2>0$:
$$\sup_{c(f(x),x)\in\cC}\Big|\bE_{(x,y)\sim \cD}[c(f(x),x)s(f(x),y)]-\bE_{(x,y)\sim S_m}[c(f(x),x)s(f(x),y)]\Big|\le \alpha.$$
\end{definition}

Examples of such upper bounds for this dimension include metric entropy for $\cC$.

Given the above preparations, let us provide the following sample version algorithm. 

\begin{algorithm}[H]{
\caption{HappyMap}\label{alg:samplegmc}
\KwIn{Step size $\eta>0$, bound $T\in\bN_+$ on number of iterations, initial predictor $f_0(\cdot)$, distribution $\cD$, convex set $O\subseteq \bR$, mapping $s:\bR \times {\mathcal Y}\rightarrow \bR$, $T$ validation datasets $D_0,\cdots,D_T$, each one with $m$ samples }
}
Set $t=0$

\While {$t < T$ ~and $~ \exists c_t\in\cC: \bE_{(x,y)\sim D_t}[c(f_t(x),x)s(f_t(x),y)]> 3\alpha/4$}
{
\vspace{0.05in}
Let $c_t$ be an arbitrary element of $\cC$ satisfying the condition in the while statement

$\forall x\in \cX$,~$f_{t+1}(x)=\Pi_O\big[f_t(x)-\eta\cdot c(f(x),x)\big]$ 
\\
$t=t+1$}

\KwOut{$f_t(\cdot)$}
\end{algorithm}

\begin{theorem}\label{thm:sample}
Under our assumptions, if $\cC$ is symmetric, suppose we run Algorithm~\ref{alg:samplegmc} with a suitably chosen $\eta=\cO(\alpha/(\kappa_\cL B))$ and $m=\Omega(T\cdot\frac{d(\cC)+\log(1/\delta)}{\alpha^2})$, then with probability at least $1-\delta$, the algorithm converges in $T=\cO((C^u_\cL-C^l_\cL) \kappa_\cL B/\alpha^2)$, which results in $$\Big|\bE_{(x,y)\sim \cD}[c(f(x),x)s(f(x),y)]\Big|\le\alpha$$
for the final output $f$ of Algorithm~\ref{alg:samplegmc}.
\end{theorem}
\begin{proof}
We can take a suitably chosen $m=\Omega(T\cdot\frac{d(\cC)+\log(1/\delta)}{\alpha^2})$, such that for all $t\in [T]$, 
$$\Big|\bE_{(x,y)\sim \cD}[c(f_t(x),x)s(f_t(x),y)]-\bE_{(x,y)\sim D_t}[c(f_t(x),x)s(f_t(x),y)]\Big|\le\alpha/4.$$

Thus, whenever Algorithm~\ref{alg:samplegmc} updates, we know 
$$\bE_{(x,y)\sim \cD}[c(f_t(x),x)s(f_t(x),y)]\ge \alpha/2.$$
Thus, the progress for the underlying potential function is at least
$\frac{\alpha^2}{8\kappa_\cL B}$. Following similar proof of Algorithm~\ref{alg:gmc}, as long as $T$ satisfying $(C^U_\cL-C^l_\cL)/\frac{\alpha^2}{8\kappa_\cL B}<T$, we know Algorithm~\ref{alg:samplegmc} provides a solution $f$ such that 
$$\Big|\bE_{(x,y)\sim \cD}[c(f(x),x)s(f(x),y)]\Big|\le\alpha.$$
\end{proof}

\begin{remark}
The sample version results for all the theorems in Section 4,5 can naturally follow the above theorem. We will not reiterate that in our paper.
\end{remark}

\section{Omitted Details for Section~\ref{sec:application}}\label{app:application}

\subsection{Proof of Remark~\ref{rem:twice}}
\begin{proof} Since $\1\{y\in\Cint(x)\}=\1\{y\le l_{\delta/2}(x)\}-\1\{y\le l_{1-\delta/2}(x)\}$ and $(1-\delta)=(1-\delta/2)-\delta/2$,  we have $$\1\{y\in\Cint(x)\}-(1-\delta)=[\1\{y\le l_{\delta/2}(x)\}-\delta/2]-[\1\{y\le l_{1-\delta/2}(x)\}-(1-\delta/2)].$$ Therefore,  $\sup_{c\in\mathcal C}\E[ c(x)(\1\{y\in\Cint(x)\}-(1-\delta))]\le\sup_{c\in\mathcal C}\E[ c(x)(\1\{y\le l_{\delta/2}(x)\}-\delta/2]+\sup_{c\in\mathcal C}\E[ c(x)(\1\{y\le l_{1-\delta/2}(x)\}-(1-\delta/2))]\le 2\alpha$
\end{proof}

\section{Omitted Details for Section~\ref{sec:ua}}\label{app:ua}

\subsection{Proof of Theorem~\ref{thm:acctransfer}}\label{sec:app:acctransfer}
\begin{proof}
$\cC=\{c(f(x),x)=\pm\frac{1-\sigma(x)}{\sigma(x)}(f(x)-p(x)):\sigma\in\Sigma,p\in\cP\}$, where $\sigma,p\in(0,1)$ and$\cL(f(x),y)=1/2(f(x)-y)^2$ and $O=[0,1]$. Then, we directly apply the proof of Theorem~\ref{thm:gmc}, and we can obtain for all $\sigma\in\Sigma$ and $p\in\cP$,
$$\Big|\bE_{(x,y)\sim \cD}[\frac{1-\sigma(x)}{\sigma(x)}(f_t(x)-p(x))(f_t(x)-y)]\Big|\le\alpha.$$

 To obtain the final guarantee, for any $\hat f$, let us denote $\delta_{\sigma,\hat f,p}=\phi^*_{\hat f}(x)-\phi_{\sigma,\hat f,p}(x)$, where $\phi^*_{\hat f}$ is the underlying truth $\frac{1-e(x)}{e(x)}(\hat f(x)-\mathbb{E}(y|x))$ and $\phi_{\sigma,\hat f,p}(x)=\frac{1-\sigma(x)}{\sigma(x)}(\hat f(x)-p(x))$.

  \begin{align*}
    \mathbb{E}_{(x,y)\sim\cD_{ta}}[\left(\hat f(x)-\mathbb{E}(y|x)\right)^2]&=\mathbb{E}_{(x,y)\sim \cD_{so}}[\frac{1-e(x)}{e(x)}(\hat f(x)-\mathbb{E}(y|x))(\hat f(x)-y)]\\
    &=\mathbb{E}_{(x,y)\sim \cD_{so}}(\phi_{\sigma,\hat f,p}(x)(\hat f(x)-y))+\mathbb{E}_{s}((\phi^*_{\hat f}(x)-\phi_{\sigma,\hat f,p}(x))(\hat f(x)-y))\\
    &\le \alpha + \sqrt{\mathbb{E}_{(x,y)\sim \cD_{so}}(\hat f-y)^2\mathbb{E}_{(x,y)\sim \cD_{so}}\delta^2_{\sigma,\hat f,p}}.
    \end{align*}
  
  We should notice that we can choose  suitable $\sigma\in\Sigma, p\in \mathcal P$ and the above bound can be written as:
  $$ \mathbb{E}_{(x,y)\sim \cD_{ta}}[\left(\hat f(x)-\mathbb{E}(y|x)\right)^2]\le \alpha + \min_{\sigma\in\Sigma,p\in \mathcal P}\sqrt{\mathbb{E}_{s}(\hat f-y)^2\mathbb{E}_{(x,y)\sim \cD_{so}}\delta^2_{\sigma,\hat f,p}}$$
  
  Now, the only task is to bound $\min_{\sigma\in\Sigma,p\in \mathcal P}\mathbb{E}_{(x,y)\sim \cD_{so}}\delta^2_{\sigma,\hat f,p}$. Let us denote the bias $$\beta_{st}=\min_{\sigma\in\Sigma}\sqrt{\mathbb \mathbb{E}_{(x,y)\sim \cD_{so}}(\frac{1-\sigma(x)}{\sigma(x)}-\frac{1-e(x)}{e(x)})^4}.$$

  Thus, we have 
 $$\min_{\sigma\in\Sigma,p\in\mathcal P}\mathbb{E}_{(x,y)\sim \cD_{so}}\delta^2_{\sigma,\hat f,p}\le 2\beta_{st}\sqrt{\mathbb{E}_{(x,y)\sim \cD_{so}}(\hat f-\mathbb{E}(y|x))^4}+2\min_{p\in\mathcal P}\sqrt{\mathbb{E}_{(x,y)\sim \cD_{so}}[\frac{1-\bar{\sigma}(x)}{\bar{\sigma}(x)}]^4}\sqrt{\mathbb{E}_s[p(x)-\mathbb{E}(y|x)]^4}$$
  where $\bar\sigma$ is the one reach the min of $\beta_{st}$. So, combining with the fact that $f,p,y\in [0,1]$, we can obtain the final result.
\end{proof}

\subsection{Proof of Theorem~\ref{thm:one-sided}}\label{sec:one-sided}
\begin{proof}
  Following Theorem~\ref{thm:gmc} and proof of Theorem~\ref{thm:conformal} (which shows that the function $s(l,y)=\1\{l\le y\}-(1-\delta)$ satisfies the conditions (b)-(c) of Theorem~\ref{thm:gmc}), we have 
$$
\Big|\sup_{c\in\cC}\bE_{(x,y)\sim \cD_{so}}[c(x)\cdot(\1\{l(x)\le y\}-(1-\delta))]\Big|\le \alpha.
$$
Since    $\Prob_{(x,y)\sim \cD_{ta}}(l(x)\le y)= \bE_{(x,y)\sim \cD_{so}}\Big(\frac{1-e(x)}{e(x)}\1\{l(x)\le y\}\Big)$, we then have \begin{align*}
    |\Prob_{(x,y)\sim \cD_{ta}}(l(x)\le y)-(1-\delta)|\le&\alpha+|\bE_{(x,y)\sim \cD_{so}}[(c(x)-\frac{1-e(x)}{e(x)})\cdot(\1\{l(x)\le y\}-(1-\delta))] |\\
    \le& \alpha+\sqrt{\bE_{(x,y)\sim \cD_{so}}[(c(x)-\frac{1-e(x)}{e(x)})^2]\cdot \bE_{(x,y)\sim \cD_{so}}[(\1\{l(x)\le y\}-(1-\delta))^2]}\\
    \le&\alpha+\beta,
\end{align*}
where the last inequality uses the fact that $|\1\{l(x)\le y\}-(1-\delta)|<1$.
\end{proof}

\subsection{Proof of Theorem~\ref{thm:missing}}\label{sec:app:missing}
\begin{proof}
  Following Theorem~\ref{thm:gmc} and the fact that the potential function $\cL(f(x),y)=1/2(f(x)-y)^2$ satisfies the conditions (b)-(c) of Theorem~\ref{thm:gmc}), we have 
$$        \Big|\sup_{c\in\cC}\bE_{(x,y)\sim \cD_{so}}[c(x)\cdot(y-f(x))]\Big|\le \alpha.
$$
Since $\E_{(x,y)\sim \cD}[(y-f(x))^2]$ can be equivalently written as $\E_{(x,y)\sim \cD\mid R=1}[\frac{\bP(x\mid R_i=1)}{\bP(x)}(\E[y| x]-f(x))\cdot(y-f(x))]$, we then have \begin{align*}
    \E_{(x,y)\sim \cD}[(y-f(x))^2]\le&\alpha+|\bE_{(x,y)\sim \cD_{so}}[(c(f(x),x)-\frac{\bP(x\mid R_i=1)}{\bP(x)}(\E[y| x]-f(x)))\cdot(y-f(x))] |\\
    \le& \alpha+\sqrt{\bE_{(x,y)\sim \cD_{so}}[(c(x)-\frac{\bP(x\mid R_i=1)}{\bP(x)}(\E[y| x]-f(x)))^2]\cdot \bE_{(x,y)\sim \cD_{so}}[(y-f(x))^2]}\\
    \le&\alpha+\beta,
\end{align*}
where the last inequality uses the fact that after projection to $O=[0,1]$, $|f(x)-y|<1$.
\end{proof}

\remove{
\subsection{Details for the Sample Version Algorithms}\label{app:sample version ua}

\begin{theorem}\label{thm:one-sided}
If $\E[c^2(x)]\le B$, and take $s(f,y)=\1\{f\le y\}-(1-\delta)$ in Algorithm~\ref{alg:gmc}. In addition, suppose that $\|c\|_\infty<B$ for all $c\in\mathcal C$, and the $(\alpha/4)$-metric entropy of $\mathcal C$ under the metric $\|f\|=\sqrt{\bE[f(x)^2]}$ is $d(\mathcal C)$, and the sample size $n\gtrsim ( \frac{(d(\mathcal C)+\log(1/\epsilon))LB}{\alpha^4})$. Then for any target distribution whose marginal density function of $x$ $p_t(x)$ satisfies $\sup_{c\in\mathcal C}\|c-\frac{p_t(x)}{p_s(x)}\|\le\beta$, there exists a choice of $\eta>0$ such that with probability $1-\epsilon$, Algorithm 1 terminates in $T=O(LB/\alpha^2)$ will produce a function $l_{\delta}$ that satisfies
    $$
    \|\Prob_t(l_{\delta}(x)\le y)-(1-\delta)\|\le\alpha+\beta.
    $$
\end{theorem}
}

\section{Relation to the existing literature}
\subsection{Relation to Omnipredictors}\label{sec:exist}

In a recent work, \cite{gopalan2021omnipredictors} proposed to use a notion of multi-calibration to create an $(\mathcal L, \mathcal C)$ omnipredictor--- a predictor that could be used to optimize any loss in a family $\mathcal L$, that is, the output of such a predictor can be post-processed to produce a small loss compared with any hypothesis from a given function class $\mathcal G$. 

We first  state the definition of multi-calibration in \cite{gopalan2021omnipredictors}. 
\begin{definition} (\cite{gopalan2021omnipredictors})
\label{def:gopalan}
Let $\cD$ be a distribution on $\mathcal X \times\{0, 1\}$. The partition $\mathcal S=\{S_1,...,S_m\}$ of $\mathcal X$ is $\alpha$-multicalibrated for $\mathcal G$, $\mathcal D$ if for every $i \in [m]$ and $c \in \mathcal G$, the conditional distribution $\cD_i = \cD\mid x \in S_i$ satisfies \begin{equation}\label{mc:omni}
    |\Cov_{\cD_i}(c(x),y)|\le\alpha.
\end{equation}
\end{definition}

While this notion is seemingly different from the Definition~\ref{def:GMC}. \cite{gopalan2021omnipredictors} showed that when $c(x)$ are Boolean functions, Definition~\ref{def:gopalan} is equivalent to  mean-multicalibration \cite{jung2021moment}. Our next theorem further extends this relationship to real-valued function $c$, and shows that our generalized notion in Equation~\eqref{eq:1} in fact implies \eqref{mc:omni}.

\begin{theorem}
Consider a function class $\mathcal G$. Denote
$\Lambda=\{\lambda, 3\lambda, 5\lambda,  ...,1-\lambda]\}$
and $I_v=[v-\lambda,v+\lambda]$. Take $\mathcal I_\Lambda=\{I_v\}_{v\in \Lambda}$ and we let $\mathcal C$ include $\1\{f(x)\in I_v\}$ and $c(x) \1\{f(x)\in I_v\}$ for $c\in\mathcal G$. Assuming $|c(x)|<L$ and $\Prob(x\in I_v)\ge\gamma$.

If a predictor $f$  satisfies $\alpha$-multi-calibration with respect to $\mathcal C$ in Equation~\eqref{eq:1}  with $s(f(x),y)=f(x)-y$, then the set partition $\mathcal S_\Lambda$ defined by this predictor $f$ satisfies $(\frac{(L+1)\alpha}{\gamma}+\lambda)$-multi-calibration with respect to $\mathcal G$ in Equation~\eqref{mc:omni}. \end{theorem}

\begin{proof}
Recall that $\mathcal G$ include $ \1\{f(x)\in I_v\}$ and $c(x) \1\{f(x)\in I_v\}$. Let us define $S_v=\{x:f(x)\in I_v\}.$ Then \eqref{def:GMC} implies \begin{align*}
     |\E[\1\{x\in S_v\}\cdot (f(x)-y)]|&<\alpha,\\
|\E[\1\{x\in S_v\} c(x)\cdot (f(x)-y)]|&<\alpha.
\end{align*}
Then
\begin{align*}
&|\E[\1\{x\in S_v\}c(x)\cdot(y-\E[y\mid x\in S_v]) ]|\le\alpha+|\E[\1\{x\in S_v\}c(x)\cdot(f(x)-\E[y\mid x\in S_v]) ]|\\
\le&\alpha+|\E[\1\{x\in S_v\}c(x)\cdot(\E[f(x)-y\mid x\in S_v]) ]|+|\E[\1\{x\in S_v\}c(x)\cdot(f(x)-\E[f(x)\mid x\in S_v]) ]|
\end{align*}
For the second term, we have $$
|\E[\1\{x\in S_v\}c(x)\cdot(\E[f(x)-y\mid x\in S_v]) ]|\le |\E[\1\{x\in S_v\}|c(x)|\cdot\frac{\alpha}{\Prob(x\in S_v)} ]|\le\alpha L.
$$
For the third term, since $|f(x)-\E[f(x)\mid x\in S_v]|<\lambda$ for any $x\in S_v$, we have $$
|\E[\1\{x\in S_v\}c(x)\cdot(f(x)-\E[f(x)\mid x\in S_v]) ]|\le L\lambda \Prob(x\in S_v).
$$
Combining all the pieces, we have $$
|\E[\1\{x\in S_v\}c(x)\cdot(y-\E[y\mid x\in S_v]) ]|\le (L+1)\alpha+L\lambda \Prob(x\in S_v),
$$
implying
\begin{align*}
     |\Cov(c(x),y\mid x\in S_v) |
   =&|\E[c(x)(y-\E[y\mid x\in S_v])\mid x\in S_v]|\\
   =&\frac{|\E[c(x)(y-\E[y\mid x\in S_v])\cdot \1\{x\in S_v\}]|}{\Prob(x\in S_v)}\\
   \le&\frac{(L+1)\alpha}{\gamma}+\lambda.
\end{align*}
\end{proof}
\subsection{Generalized Statistical Parity Subgroup Fairness}\label{sec:gerrymandering}

Kearns {\it et al.} \cite{kearns2018preventing} consider multi-parity (also called statistical parity subgroup fairness), in which, given a pre-specified collection of possibly intersecting demographic subgroups, the demographics of those assigned a positive outcome are the same as the demographics of the population as whole.  
In this section, we show that the constraints of statistical parity and equality of false positive (and false negative) rates can be expressed by appropriate choice of mappings in \HM .  
The multi-group versions of these guarantees can be achieved trivially via a constant function.  In this section, in contrast to \cite{kearns2018preventing}, we do not give guarantees for our algorithms.  We state our observations in anticipation of extensions of \HM{} to constrained optimization.

Following \cite{kearns2018preventing}, we aim to find randomized classifiers that satisfy this fairness notion. We consider binary classification, where $\mathcal Y=\{0,1\}$. We let $\mathcal{U}(x)=\1\{U<f(x)\}$ for uniform random variable $U\sim U(0,1)$, where $f(x)\in[0,1]$ denotes a fitted probability function such that $\cU(x)\sim Bern(f(x))$. Suppose there is a collection of demographic group indicator functions $\mathscr{G}=\{g:\mathcal X\to\{0,1\}$, where $g(x)=1$ indicates that an individual with features $x$ is in group $g$. The multi-parity requirement  says that, for all $g\in\mathscr{G}$,
\begin{equation}
    |\Prob(g(x)=1)\cdot(\Prob(\cU(x)=1)-\Prob(\cU(x)=1\mid g(x)=1))|\le\alpha.
\end{equation}

This inequality can be rewritten as $$
|\E[(\1\{g(x)=1\}-\Prob(g(x)=1))\cdot\1\{U<f(x)\} ]|\le\alpha.
$$

This notion only considers fairness auditors that are indicator functions where the range of $g$ is $\{0,1\}$.  
In the following, we generalize multiparity to allow real-valued $g$'s:
\begin{equation}\label{eq:mp:extend}
        \sup_{g\in\mathcal G}\Big|\E[ (g(x)-\E[g(x)])\cdot\1\{U<f(x)\}]\Big|\le\alpha,
\end{equation}
where $\mathcal G\subset [0,1]^{\mathcal X}$ denotes a collection of auditors  with possibly continuous ranges. As suggested by \cite{kim2019multiaccuracy}, this allows the inclusion of more flexible statistical tests.
By taking $c(f(x),x)=g(x)-\E[g(x)]$, and $s(f,y)=\1\{U<f\}$,  an application of \HM{} 
to this setting (with the slight modification that the \textbf{while} condition should check  $\exists g\in\cG: \bE_{U,(x,y)\sim \cD}[(g(x)-\E[g(x)])\cdot(\1\{f_t(x)> U\})]> \alpha$, where the expectation is taken over the randomness of both the data and $U$)  yields the following guarantee. 
\begin{theorem}\label{thm:mp}
Suppose that $\bE_{x\sim \cD^\cX}[c^2(x)]\le B$ for all $c\in\mathcal C$. Taking a suitably chosen $\eta=\cO(\alpha/B)$, then \HM{} (Algorithm~\ref{alg:gmc}), when run with $T=\cO( B/\alpha^2)$ and the potential function $\cL(f(x),y)=(1-\delta)\cdot f(x)-\E_U[\min(f(x)-U,0)]$ outputs a function $l(\cdot)$ satisfying
$$
\sup_{g\in\mathcal G}\E[ (g(x)-\E[g(x)])\cdot\1\{U<f(x)\}]\le\alpha.
$$
\end{theorem}
\begin{proof}[Proof of Theorem~\ref{thm:mp}]
The proof of Theorem~\ref{thm:mp} directly follows the proof of Theorem~\ref{thm:conformal}, by replacing the $y$ there by $U$. The same derivation still apply to this case, and we omit the details here. 
\end{proof}
\begin{remark}
We can similarly generalize the false positive subgroup fairness notion considered in \cite{kearns2018preventing}, that is, $
|\Prob(g(x)=1, y=0)\cdot(\Prob(\cU(x)=1)-\Prob(\cU(x)=1\mid g(x)=1))|\le\alpha,
$ to the case where we consider all groups that can be identified by certain classifiers:
$        \sup_{g\in\mathcal G}\Big|\E[ (g(x)-\E[g(x)\mid y=0])\cdot\1\{U<f(x)\}\mid y=0]\Big|\le\alpha.
$ Applying \HM{} 
to the distribution $\mathcal D\mid y=0$ yields a classifier that satisfies this constraint.
\end{remark}

\subsection{Comparison between marginal coverage and calibrated coverage}\label{sec:comp}

In this section, we compare the notion in Eq.~\eqref{eq:conformal:extend} and multivalidity, in the setting where $\mathcal C=\{1\}$. In this special case, we call them
 marginal coverage and calibrated coverage respectively. We first introduce some basics and definitions as below.

 Suppose $y$ is a continuous random variable. For $\delta>0$, we aim to construct the $(1-\delta)$ prediction interval $C_n(x)=(l(x), u(x))$ based on $n$ i.i.d. samples $(x_i,y_i)\sim P_{x,y}$, the joint distribution of $(x,y)$. Following the set-up in \cite{gupta2021online, bastani2022practical}, the prediction interval $C_n(x)$ is constructed based on a non-conformity score $m(x,y)$, eg. $|y-h(x)|$ for an arbitrary prediction function $h$, and a conformity threshold function $q(x)$. Then $C_n(x)=\{y: m(x,y)<q(x)\}$.
 
The marginal coverage guarantee asks that $C_n(x)$  satisfies \begin{equation}\label{eq:marginal}
        \Prob(y\in C_n(x))=1-\delta,
    \end{equation}
    for all $P_{x,y}$, {where the probability is taken over all training samples and $(x,y)\sim P_{x,y}$}. 
        
    The calibrated coverage guarantee asks that $C_n(x)$ satisfy \begin{equation}\label{eq:calibrated}
    \Prob(y\in C_n(x)\mid q(x)=q)=1-\delta,
        \end{equation}
        for all $P_{x,y}$ and almost all $q$.
We then have the following theorem. 
\begin{theorem}\label{thm:length}
Suppose that an interval $C_n(x)$ {constructed based on $n$ samples}
has $1-\delta$ calibrated coverage in the sense of condition~\eqref{eq:calibrated}. Then for any $P_{x,y}$, with probability 1,
$$
|C_n(x)|=\infty,
$$
at almost all points $q$ that $q(x)=q$ is non-atom of $P_{q}$, where $P_q$ denotes the marginal distribution of $q(x)$.

Here, a point $q$ is a non-atom for $P_q$ if $q$ is in the support of $P_q$ and if $\Prob(B(q,\epsilon))\to 0$ as $\epsilon\to 0$, where $B(q,\epsilon)$ is the Euclidean ball centred at $q$ with radius $\epsilon$.
\end{theorem}
We note that the key part of this Theorem is the ``finite-sample'' and view $C_n(\cdot)$ as an interval produced by applying a certain method to the training samples. When a certain method produces an interval from $n$ samples and needs to satisfy \eqref{eq:calibrated} for a large class of distributions, Theorem~\ref{thm:length} suggests such an interval has to have infinite length. In contrast, the standard conformal prediction methods, such as the split conformal \cite{shafer2008tutorial}, produces a finite-length interval that satisfies Equation~\eqref{eq:marginal} for all $P_{x,y}$.
\begin{remark}
In \cite{bastani2022practical}, some regularity conditions on $P_{x,y}$ are imposed, making the class {of} distributions smaller and therefore Theorem~\ref{thm:length} does not apply. 
This suggests that some regularity conditions are necessary for calibrated coverage, but not for marginal coverage. Furthermore, 
 under certain additional regularity conditions, calibrated coverage would produce short intervals and rule out long ones. For example, let us assume $y\in[-1,1]$. Taking the nonconformity score $m(x,y)=|y|$, we have that all intervals that satisfy \eqref{eq:calibrated} must have length less than 2. This is due to the fact that $q(x)$ has to be smaller than 1, otherwise conditioning on $q(x)=q$ for $q\ge 1$ would make the left hand side of \eqref{eq:calibrated} equal to 1. In contrast, the intervals that satisfy \eqref{eq:marginal} might have infinite length. For example, taking $C_n(x)=\R$ for $(1-\delta)\times 100\%$ samples of $x$, and $C_n(x)=0$ otherwise. Such a $C_n(x)$ satisfies \eqref{eq:marginal}, but not the stronger requirement \eqref{eq:calibrated}. 
\end{remark}
\begin{remark}
We would also like to point out that Theorem~\ref{thm:length} only holds for all non-atomic points in $P_q$. It suggests that discretization of $q(x)$ is necessary to achieve calibrated coverage (or multivalid) under weak assumptions. \end{remark}

\begin{proof}[Proof of Theorem~\ref{thm:length}] The key idea of this proof follows \cite{lei2014distribution} and \cite{foygel2021limits}. Let us prove the theorem for any given distribution $P_q$. For ease of presentation, we omit the subscript $q$ and write the distribution as $P$. We first fix arbitrary $\epsilon, B>0$. 

For a non-atom point $x_0$, there exists a $\delta_0>0$ such that $\Prob_P(q(x)\in B(q(x_0),\delta_0)<\epsilon_n$ where $\epsilon_n=2\{1-(1-\epsilon^2/8)^{1/n}\}$. We also let $B_0=\frac{B}{2(1-\delta)}$.

We then define a new distribution $Q$, by letting $$
\Prob_Q(A)=\Prob_P(A\cap S^c)+\Prob_U(A\cap S),
$$
where $S=\{(x,y): q(x)\in B(q(x_0),\delta_0)\}$ and $U$  is uniform in $\{(x,y): q(x)\in B(q(x_0),\delta_0), |y|\le B_0\}$ and has total mass $\Prob_P(S)$.

By definition, we can verify $TV(P,Q)\le\epsilon_n$, and therefore $TV(P^n, Q^n)\le\epsilon$.

Since \eqref{eq:calibrated} also holds for $Q$, we then have $$
\int_{C_n(x)} dQ(y\mid q(x)) \ge 1-\delta,
$$
for all $x$ such that $q(x)\in B(q(x_0),\delta_0)$. For such $x$'s, we have $|dQ(y\mid q(x))|\le \frac{1}{2B_0}$, implying $|C_n(x)|\ge 2B_0(1-\delta)=B$. Therefore, for such $x$'s $$
\Prob_Q(|C_n(x)|>B)=1,
$$
and it follows that $$
\Prob_P(|C_n(x)|>B)\ge \Prob_Q(|C_n(x)|>B)-\epsilon= 1-\epsilon.
$$

Since $\epsilon$ and $B$ are arbitrary, we complete the proof. 
\end{proof}

\section{Function Class Complexity Reduction}
\label{sec:reduction}

In this section, we briefly discuss how to reduce the complexity of the class $\cC$ in some cases. We use tools from representation learning recently developed by the machine learning community \cite{tripuraneni2020theory,ji2021unconstrained,ji2021power,nakada2023understanding,ye2022freeze}, which was used to improve generalization \cite{kawaguchi2022understanding,yao2022improving,yao2021improving,yao2021meta, yao2022c}, robustness \cite{deng2021adversarial}, and fairness \cite{deng2020representation,burhanpurkar2021scaffolding}. Indeed, in our applications, if $\cC$ is a very large class, it will result in the inefficiency of Algorithm~\ref{alg:gmc} since we need to search for $c_t$ in a rich class at each iteration. 

Specifically, we will focus on the case for using $\cC$ to capture the propensity score ratio
$$\frac{1-e(x)}{e(x)}=\frac{\bP(z=ta|x)}{\bP(z=so|x)}.$$
Here, $\cC$ is taken as $\cC=\{\frac{1-\sigma(x)}{\sigma(x)},\sigma\in\Sigma\}$ and $\sigma:\cX\mapsto [0,1]$. We remark here that there exists a dilemma in choosing $\Sigma$. On one hand, if $\Sigma$ is not rich enough, then the class may not be expressive enough and it is highly possible that $e(x)\notin\Sigma$. On the other hand, if $\Sigma$ is too rich, Algorithm~\ref{alg:gmc} may face inefficiency problem. Is there a way to make $\Sigma$ expressive while keep the size of $\Sigma$ moderate? In this section, we explore the possibility to use pre-trained models to achieve expressiveness while keeping the class size moderate. 

Specifically, the gist of using pre-trained model is that for a function class  $\cG=\{g=w\circ h,w\in\cW,h\in\cH\}$, where $\cW$ is of low complexity and $\cH$ is of high complexity, we can use some related and abundant data set to pre-train $h$ and obtain $\hat h$. Then, when it comes to the data set $D$ we truly want to apply $g$ to, we only need to select the final $g$ from the class $\cG_{shrink}=\{g=w\circ \hat{h},w\in\cW\}$. This is quite a common scheme in transfer learning \cite{tripuraneni2020theory}.

\noindent\textbf{Pre-training setup.} Now, we demonstrate how to use transfer learning to pre-train $h$. For simplicity, let us consider $g(x)=w^Th(x)$, that is $g$ is a composed of a linear transformation and a possibly complex function $h$. This is quite common in modern learning, where $h$ is usually called representation. If $h$ is pre-trained, it is called fix-feature learning. Let us consider $\sigma(x)=e^{g(x)}/(1+e^{g(x)})$ so that $\Sigma =\{\frac{e^{g(x)}}{1+e^{g(x)}}:g=w^T\hat h,w\in\mathcal {W}\}$, where $\hat h$ is a learned representation via transfer learning.

To use the idea of transfer learning, we assume 
$$e(x)\in \Sigma_{h^*}:=\{\frac{e^{g(x)}}{1+e^{g(x)}}:g=w^Th^*,w\in\mathcal {W}\}.$$
To obtain $\hat h$, consider there are $k$ mutually independent source-target data set pairs $\{(D^{so}_{j},D^{ta}_j)\}_{j=1}^k$ and $D^{so}_j$ and $D^{ta}_j$ are also independent. $D^{so}_j$ consists of pair of form $(x,1)$ and $D^{ta}_j$ consists of pair of form $(x,0)$, here $0,1$ are labels. Let us denote the distribution of the augmented dataset $\{(D^{so}_{j},D^{ta}_j)\}_{j=1}^k$ as $\{\cD_i\}_{i=1}^k$. The propensity score ratios of $\{(\cD_i\}_{i=1}^k$ all share the same representation $h^*$ in the sense that $\cD_{j}\sim Bern(\rho(w_{j}^T h^{*}))$, which means
$$\bP(z=so|x)=\rho(w_j^Th^*(x)),$$
where $\rho(x)=e^x/(1+e^x)$.

To further understand why such setting is reasonable, as a concrete illustration in real life, let us consider the source data is coming from a hospital in Boston and the target is to apply the trained model to a hospital in San Francisco. We can choose data from some other hospitals from Boston and data from some hospitals from  SF and pair them as source-target pairs. Usually, in meta-learning, people can use those source-target data pairs to train models sharing the same representation with additional pair specific simple structures for each pair. This works well in practice. That motivates a series theoretical work based on the model setting we consider here.

Without loss of generality, we assume all data set pairs $(D^{so}_j,D^{ta}_j)$ have $n$ samples in total, and we denote the samples in $(D^{so}_j,D^{ta}_j)$ as $\{(x_{ji},y_{ji})\}_{i=1}^n$. Here, the label $y_{ji}$'s are binary, so we can be viewed as obtaining $\rho(w_j^Th^*(x))$ as classification problem, and a commonly used loss is the cross entropy loss $\ell(z,y)=-y\log(\frac{e^z}{1+e^z})$. So, we can obtain $\hat h$ from $\cH$ by the following optimization:

$$(\{\hat w_{j}\}_{j=1}^j,\hat h)=\argmin_{\{w_j\},h\in\cH}\frac{1}{nk}\sum_{j=1}^k\sum_{i=1}^n \ell(w_j\circ h(x_{ji}),y_{ji}).$$

\noindent\textbf{Derivation.} Let us first introduce some notations and definitions. The derivation mainly follows \cite{tripuraneni2020theory}, we include all details here for completeness.

\begin{definition}[Gaussian complexity and variants]
For a generic function class $\cQ$ containing $q(\cdot)$, and $N$ data points, $\bar{X}=(x_1,\cdots,x_N)^T$, the empirical Gaussian complexity is defined as 
$$\hat {\mathscr{O}}_{\bar{X}}(\cQ)=\bE_g[\sup_{q\in\cQ}\frac{1}{N}\sum_{k=1}^r\sum_{i=1}^Ng_{ki}q_k(x_i)],~g_{ki}\sim\cN(0,1)~i.i.d.,$$
where $g=\{g_{ki}\}_{k\in[r],i\in[N]}$, and $q_k(\cdot)$ is the $k$-th coordinate of the vector-valued function $q(\cdot)$. The corresponding population Gaussian complexity as $\sO_N(\cQ)=\bE_{\bar{X}}[\hat{\sO}_{\bar{X}}(\cQ)]$, where the expectation is taken over the distribution of data samples $\bar{X}$.
\end{definition}

\begin{definition}
For a function class $\cW$, $k$ vectors $w=(w_1,\cdots,w_k)^T$, and data $(x_j,y_j)\sim Bern(\rho(w_j^Th^*(x_j)))$, let us define
$$\bar{d}_{\cW,w}(h,h')=\frac{1}{k}\sum_{j=1}^k\inf_{w'\in\cW}\bE_{(x_j,y_j)}\{\ell(w'^Th'(x_j),y_j)-\ell(w_j^Th^*(x_j),y_j)\}.$$
\end{definition}

We first provide a theoretical guarantee of $\hat h$ under $\bar{d}_{\cW,w}(\hat h,h^*)$. In order to do that, let us further define \textit{worst-case Gaussian complexity} over $\cW$:
$$\bar{\sO}_n(\cW)=\max_{Z\in\cZ}\hat{\sO}_Z(\cW),$$
where $\cW=\{h^*(x_1),h^*_2(x_2),\cdots,h^*(x_n)|h^*\in\cH,x_i\in\cX$ for all  $i\in[n]\}$. Also, we let $\sO_{nk}=\bE_{X}[\hat{\sO}_X(\cQ)]$ to denote the population Gaussian complexity computed with respect to the data matrix $X$ formed from the concentration of the $nk$ training data points $\{x_{ji}\}_{j.i}$.

\noindent\textbf{Assumption.} For all $w\in\cW$, $\|w\|\le L$. $\cD_{i}\sim Bern(w_{i}^T h^{*})$, where $w_{i}\in \cW$, $h^*\in \cH$. $x_{ji}\in \cX$. We further assume for any $h\in\cH$, $h(x)$ is bounded for all $x\in\cX$, i.e., $\sup_{x\in\cX}\|h(x)\|\le B_h$. Based on that, we know $|w^T h|\le \|w\|\|h\|\le LB_h$.

\begin{theorem}\label{thm:dist}
With probability $1-\delta$,
$$\bar{d}_{\cW,w}(\hat h, h^*)\lesssim \cO\Big(\frac{LB_h}{(nk)^2}+\log(nk)[L\cdot \sO_{nk}(\cH)+\bar{\sO}_{n}(\cW)]+\sqrt{\frac{\log(2/\delta)}{nk}}\Big).$$
Here $\sO_{nk}$ and $\bar{\sO}_n$ is the population Gaussian compelxity and worst case Gaussian complexity.
\end{theorem}

\begin{proof}
The proof follows from Theorem 1 in \cite{tripuraneni2020theory}. The only thing we need to do is to verify those assumptions in Theorem 1 in \cite{tripuraneni2020theory}.

\begin{itemize}
    \item The loss function $\ell(z,y)=-y\log(\frac{e^z}{1+e^z})$ is bounded and $\ell(\cdot,y)$ is Lipchitz for all $y\in\cY$. 
    \item The function $w^\top x$ is Lipchitz with respect to $x$ in $\ell_2$ distance for all $w\in\cW$.
    \item The composed function $w^T h$ is bounded for all $w\in\cW,h\in\cH$.
\end{itemize}

Now, let us prove point by point.

First, by assumption $|w^T h|\le LB_h$, and $y\in\{0,1\}$, we know $\ell(z,y)$ is bounded for $z=w^T h(x)$ and $y\in\{0,1\}$. Then, notice $\partial_z \ell(z,,y)=\frac{-y}{1+e^z}$, and since $z$ is bounded and $y\in\{0,1\}$, we have $|\partial \ell(z,y)|\le 1$. As a result, $\ell(\cdot,y)$ is $1$-Lipchitz.

Secondly, we know that $\|w\|\le L$, so we know $w^\top x$ is $L$-Lipchitz.

Lastly, $w^T h$ is bounded by our assumption.

Then, we directly apply Theorem 1 in \cite{tripuraneni2020theory}.
\end{proof}

Now, we provide our final guarantee.

\begin{definition}
For function classes $\cW$ and $\cW_0$ such that $w_0\in \cW_0$, and data $(x,y)\sim \bP_{w_0^Th}$, the worst-case representation difference between $h,h'\in\cH$ is
$$d_{\cW,\cW_0}(h';h)=\sup_{w_0\in\cW_0}\inf_{w'\in \cW}\bE_{x,y}\{\ell({w'}^Th'(x),y)-\ell({w_0}^Th(x),y)\}.$$
\end{definition}

\begin{definition}[tasks $(\nu,\epsilon)$-diverse]
For a function class $\cW$, we say $w=(w_1,\cdots,w_k)$ are $(\nu,\epsilon)$-diverse over $\cW_0$ for $h$ if uniformly for all $h'\in\cH$,
$d_{\cW,\cW_0}(h';h)\le \bar{d}_{\cW,w}(h';h)/\nu+\epsilon$.
\end{definition}

Then, we have the final guarantee straightforwardly if we have the tasks $(\nu,\epsilon)$-diverse by Theorem~\ref{thm:dist}.
\begin{theorem}
If $h^*\in\cH$, under previous assumptions, and the training tasks are $(\nu,\epsilon)$-diverse, with probability $1-\delta$, for the unseen task $(w^*_0,h^*)$
\begin{align*}
&\inf_{w'\in \cW}\bE_{x,y}\{\ell({w'}^T\hat h(x),y)-\ell({w^*_0}^Th^*(x),y)\}\le \\
&\cO\Big(\frac{LB_h}{(nk)^2}+\log(nk)[L\cdot \sO_{nk}(\cH)+\bar{\sO}_{n}(\cW)]+\sqrt{\frac{\log(2/\delta)}{nk}}\Big)
\Big)/\nu+\epsilon.
\end{align*}
\end{theorem}

This theorem shows that how to use the abundant related data to facilitate learning $\hat h$. Thus, use the class $w^T\hat h$ for $w\in\cW$ is enough to capture the class $w^Th^*$ for $w\in\cW$.

\end{document}